\documentclass{article}

\usepackage{microtype}
\usepackage{graphicx}
\usepackage{subfigure}
\usepackage{booktabs} % for professional tables
\usepackage{dsfont}
\usepackage{enumitem}
\usepackage{rotating}
\usepackage{hyperref}

\usepackage[accepted]{icml2025}
\usepackage{multirow}
\usepackage{amsmath}
\usepackage{amssymb}
\usepackage{mathtools}
\usepackage{amsthm}

\usepackage[capitalize,noabbrev]{cleveref}

\theoremstyle{plain}
\newtheorem{theorem}{Theorem}[section]

\theoremstyle{definition}

\theoremstyle{remark}

\usepackage[textsize=tiny]{todonotes}

\icmltitlerunning{A Koopman-Kalman Enhanced Variational AutoEncoder for Probabilistic Time Series Forecasting}

\begin{document}

\twocolumn[
\icmltitle{$K^2$VAE: A Koopman-Kalman Enhanced Variational AutoEncoder for Probabilistic Time Series Forecasting}



\icmlsetsymbol{equal}{*}

\begin{icmlauthorlist}
\icmlauthor{Xingjian Wu}{equal,yyy}
\icmlauthor{Xiangfei Qiu}{equal,yyy}
\icmlauthor{Hongfan Gao}{yyy}
\icmlauthor{Jilin Hu}{yyy}
\icmlauthor{Bin Yang}{yyy}
\icmlauthor{Chenjuan Guo}{yyy}

\end{icmlauthorlist}

\icmlaffiliation{yyy}{School of Data Science and Engineering, East China Normal University, Shanghai, China}

\icmlcorrespondingauthor{Bin Yang}{byang@dase.ecnu.edu.cn}

\icmlkeywords{Machine Learning, ICML}

\vskip 0.3in
]



\printAffiliationsAndNotice{\icmlEqualContribution} % otherwise use the standard text.

\begin{abstract}
Probabilistic Time Series Forecasting (PTSF) plays a crucial role in decision-making across various fields, including economics, energy, and transportation. Most existing methods excell at short-term forecasting, while overlooking the hurdles of Long-term Probabilistic Time Series Forecasting (LPTSF). As the forecast horizon extends, the inherent nonlinear dynamics have a significant adverse effect on prediction accuracy, and make generative models inefficient by increasing the cost of each iteration. To overcome these limitations, we introduce $K^2$VAE, an efficient VAE-based generative model that leverages a KoopmanNet to transform nonlinear time series into a linear dynamical system, and devises a KalmanNet to refine predictions and model uncertainty in such linear system, which reduces error accumulation in long-term forecasting. Extensive experiments demonstrate that $K^2$VAE outperforms state-of-the-art methods in both short- and long-term PTSF, providing a more efficient and accurate solution.
\end{abstract}

\section{Introduction}
% Multivariate Time series (MTS) forecasting stands out as a critical and widely studied task~\cite{yu2024ginar,DSformer,zhu2024fcnet,HybridZheng,DBLP:conf/aaai/HuangSZCDZW24,DBLP:journals/pvldb/ChengCGZWYJ23}. It has been extensively applied in diverse domains, including economics~\cite{sezer2020financial,huang2022dgraph}, traffic~\cite{wu2024autocts++,wu2024fully,DBLP:journals/pvldb/FangPCDG21}, energy~\cite{HaoWang1,sun2022solar}, and AIOps~\cite{lin2024cocv,Monotonic,pan2023magicscaler,lin2025benchmarking}. 

\begin{figure}[!htbp]
  \centering
\includegraphics[width=0.91\linewidth]{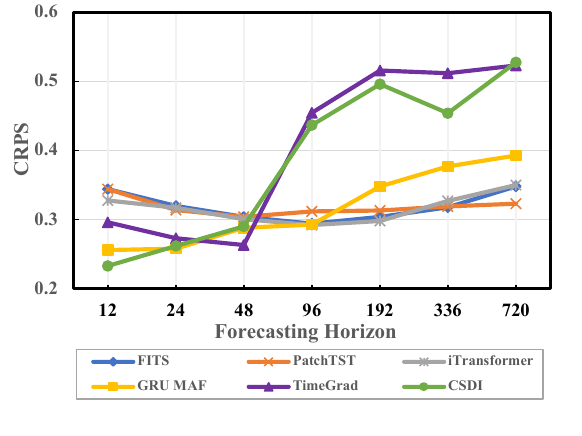}
  \caption{We compare three native probabilistic forecasting models including GRU MAF, TimeGrad, and CSDI with three point forecasting models equipped with distributional heads including FITS, PatchTST, and iTransformer on ETTh1. Longer forecasting horizons lead to rapid collapse of the CRPS metric (lower is better) on probabilistic forecasting models, even worse than point forecasting models. }
  \label{fig: intro}
\end{figure}

% However, the forecasting task is extremely challenging under the long-term setting.

% 概率时间序列预测（PTSF）是一项关键且广泛研究的任务，通过量化多个连续变量的随机时间演化，它能够为经济学、能源、交通等多个领域的决策提供重要支持。在这些实际应用中，一个迫切的需求是将预测时间延长到遥远的未来，即长序列概率时间序列预测（LPTSF), 这对于长期规划和预警非常有意义。然而，现有的方法大多是在短期问题设置下设计的，比如预测48步或更少~\cite{}。将其直接作用于长期预测任务，往往会产生较差的效果---see。作为一个实证例子，图（1）展示了在已有方法在真实数据集上的预测结果，其中Diffusion网络预测电力变压器站的小时温度，从短期（12个点，0.5天）到长期（480个点，20天）。当预测长度超过48个点（图（1b）中的实心星号）时，整体性能差距显著增大，均方误差（MSE）上升至不令人满意的水平，推理速度急剧下降，Diffusion模型开始失效。
In recent years, time series analysis has seen remarkable progress, with key tasks such as anomaly detection~\cite{D3R, liu2024elephant, miao2025parameter, hu2024multirc, wu2024catch}, classification~\cite{DBLP:conf/icde/YaoJC0GW24,DBLP:journals/pacmmod/0002Z0KGJ23}, and imputation~\cite{gao2025ssdts,escmtifs,wang2024spot,yu2025ginarp}, among others~\cite{wang2024entire,miao2024less,liu2025timecma,DBLP:conf/nips/HuangSZDWZW23,DBLP:journals/pvldb/YaoLJ00CW0G23}, gaining attention. Among these, Probabilistic Time Series Forecasting (PTSF) is a crucial and widely studied task. By quantifying the stochastic temporal evolutions of multiple continuous variables, it provides significant support for decision-making in various fields such as economics~\cite{sezer2020financial,huang2022dgraph}, traffic~\cite{wu2024autocts++,wu2024fully,pan2023ising,cirstea2022towards,DBLP:journals/pvldb/FangPCDG21,DBLP:conf/ijcai/YangGHT021,DBLP:journals/tkde/YangGY22}, energy~\cite{HaoWang1,guo2015ecomark,sun2022solar}, and AIOps~\cite{lin2024cocv,davidpvldb,Monotonic,pan2023magicscaler,lin2025benchmarking}. In these practical applications, an urgent need is to extend the prediction time to the distant future, known as Long-term Probabilistic Time Series Forecasting (LPTSF), which is highly meaningful for long-term planning and early warning. Most existing methods excell at short-term problem settings, such as predicting up to 48 steps or fewer~\cite{TimeGrad,TSDiff,normalizing-flow}, while directly applying these methods to long-term forecasting tasks often results in poor performance--see Figure~\ref{fig: intro}. 

% economics, energy, and transportation~\citep{li2024generative,dumas2022deep,huang2023metaprobformer}.
% multivariate time series forecasting~(MTSF)~\cite{yu2024ginar,DSformer,wang2024rose,zhu2024fcnet,HybridZheng,DBLP:conf/aaai/HuangSZCDZW24,DBLP:journals/pvldb/ChengCGZWYJ23} stands out as a critical and widely studied task. It has been extensively applied in diverse domains, including economics~\cite{sezer2020financial,huang2022dgraph}, traffic~\cite{wu2024autocts++,wu2024fully,cirstea2022towards,DBLP:journals/pvldb/FangPCDG21,kieu2024Team,DBLP:conf/ijcai/YangGHT021,DBLP:journals/tkde/YangGY22}, energy~\cite{HaoWang1,guo2015ecomark,sun2022solar}, and AIOps~\cite{lin2024cocv,davidpvldb,Monotonic,pan2023magicscaler,lin2025benchmarking}, highlighting its importance and impact. 

% As an empirical example, Figure (1) shows the forecasting results on a real-world dataset, where the Diffusion network predicts the hourly temperature of a power transformer station, from short-term (12 points, 0.5 days) to long-term (480 points, 20 days). When the prediction length exceeds 48 points (solid stars in Figure (1b)), the overall performance gap significantly widens, with the Mean Squared Error (MSE) rising to unsatisfactory levels, inference speed sharply declining, and the Diffusion model beginning to fail.

% 然而，在长期预测任务中，概率预测面临着诸多挑战。首先，时间序列通常具有高度的非线性特性，因为现实生活中的时间序列往往是非平稳的，且各个变量之间存在复杂的相互关系。这种非线性使得数据中的复杂模式呈现动态变化，难以通过简单的线性关系加以表达。随着时间步长的增加，序列的演化变得更加复杂，长时间跨度的预测要求概率模型能够有效捕捉这些非线性变化，并精确地建模其动态演化过程。其次，随着预测时间跨度的延长，长期预测的准确性和效率成为主要瓶颈。长期预测不仅需要更精确地建模时间序列中的长期依赖关系，还需保持高效的计算能力，以应对长序列的处理和快速训练的需求。
However, probabilistic forecasting faces numerous challenges in long-term forecasting tasks. \textit{First, the inherent nonlinearity of time series challenges probabilistic models in modeling dynamic evolution.} Due to factors such as non-stationarity and complex interdependencies between variables, time series typically exhibit nonlinear characteristics, complicating the construction of probabilistic models. Specifically, the nonlinearity makes it difficult for these models to derive a simple equation that precisely describes the state transition process. As a result, the uncertainties within the models are also hard to quantify, particularly in long-term forecasting tasks. \textit{Second, as the forecasting horizon extends, the accuracy and efficiency become major bottlenecks.} Longer foreacasting horizons lead to more intricate target distributions, which causes remarkable error accumulation. It also makes the diffusion-based~\cite{TSDiff,TimeGrad} or flow-based models~\cite{normalizing-flow} difficult to find a clear probabilistic transition path and inefficient to perform each iteration, which results in more computational consumption but worse performance.

% Long-term forecasting not only requires more precise modeling of the long-term dependencies in the time series but also needs to maintain high computational efficiency to handle long sequences and meet the demands for fast training.

% 从动力系统的角度出发，时间序列的非线性主要体现在难以从原始空间找到线性算子来刻画状态的转移过程，而自回归的非线性算子可能导致的问题是建模的误差累积且无法准确描述长步预测任务下的不确定性。Koopman Theory为非线性动力系统的建模提供了线性化方案：通过定义一组观测函数（可以有无穷个）来完善地刻画系统在每个时刻的所有观测量，并可以通过无穷维的线性Koopman Operator刻画观测量的转移过程。而Kalman Filter则为线性动力系统提供了预测精度矫正和不确定性建模，并且可以通过融合不同传感器的观测值进行实时矫正，适用于长步预测任务。这启发了我们将时间序列概率预测任务转化为：建模测量函数描述的线性动力系统的过程不确定性。

Since the nonlinearity in time series leads to the dynamic evolution of complex patterns, probabilistic models struggle to effectively capture these changes and accurately model their evolution. To tackle this thorny issue, Koopman Theory~\citep{KoopmanTheory} provides a linearization approach to transform the nonlinear time series into the space of measurement function, which is a theoretically infinite-dimeansional space characterizing all measurements of the dynamical system at each moment, and the transition process of these measurements can be captured by a linear Koopman Operator~\citep{lan2013linearization}. On the other hand, in order to accurately and efficiently model the process uncertainty and mitigate the error accumulation phenomenon in long-term forecasting, the Kalman Filter~\citep{KalmanFilter} provides a solution, which fuses observations from multiple sensors to extract the Kalman gains, to refine the prediction and process uncertainty. This inspires us to transform the probabilistic time series forecasting into modeling the process uncertainty of a linear dynamical system in the space of the measurement function.

% 在本文中，我们基于Koopman Theory和Kalman Filter提出了K^2VAE，一个基于条件生成的概率预测模型。我们将非线性的时间序列取patch划分保留足够语义后，映射到高维测量空间，并在其中应用动力系统进行预测和不确定性建模。我们首先基于eDMD和数据驱动的方法拟合了Koopman算子，构造出一个有偏的线性动力系统。 然后我们应用Kalman Filter，通过整合偏置的非线性信息作为控制输入，并将koopman算子构造的线性动力系统作为观测对象，迭代优化了预测精度并且建模了每步预测的不确定性。至此，我们成功描述了测量空间下线性动力系统预测的不确定性，这可以视为在一个具有明确语义的隐空间中建模了输入序列和高维预测序列（视为隐变量）之间的近似后验，这恰巧契合了VAE中变分推断的思想，自然地，我们从这个后验分布中采样，反向映射回原始空间构造隐变量与输出序列的似然分布，优雅地形成了K^2VAE的基本架构。通过大量实验，我们发现虽然理论上基于Diffusion的模型可以通过迭代去噪具备更强的分布拟合能力，但其易受时间序列非线性性质的影响，在长步概率预测任务上的实际表现较差，且耗时难以接受。而K^2VAE是单步生成模型，生成效率极高，且得益于Koopman Theory的非线性建模能力和Kalman Filter的误差修正以及不确定性建模能力，其能够从一个具备明确语义的后验分布中采样并且高质量和高效地完成长步和短步预测任务中的似然分布建模。贡献总结如下：

% 在本研究中，我们提出了$K^2$VAE，一种针对长序列概率时间序列预测（LPTSF）量身定制的条件生成概率预测模型。首先, 为了处理非线性并捕捉时间序列中的潜在动态，我们将时间序列分割为多个块（tokens），并将其转换为高维度的测量空间，在该空间中我们建模块之间的状态转移并评估过程的不确定性。遵循库普曼理论的假设，我们拟合库普曼算子，以构建一个“偏置”的线性动态系统。其次，为了得到更精准的长期预测性能。我们应用卡尔曼滤波器，将残差非线性信息作为控制输入，同时将偏置的线性动态系统视为观测值。通过预测并随后使用卡尔曼增益更新估计和过程不确定性，卡尔曼滤波器有效减轻了长期预测中的误差积累。与计算效率较低且可能容易受到时间序列非线性影响的扩散模型（Diffusion-based models）相比，$K^2$VAE是一个单步生成模型，具备强大的线性化和不确定性建模能力。主要贡献总结如下：

In this study, we propose $K^2$VAE, a generative probabilistic forecasting model tailored for LPTSF--see Figure~\ref{fig:data flow}. First, to handle the nonlinearity and capture the underlying dynamics in time series, we patchify the time series into tokens and model them through the KoopmanNet. The KoopmanNet provides a way to simulate the Koopman Theory, which transforms the nonlinear time series into latent measurements, and fit the Koopman Operator to construct a ``biased'' linear dynamical system easy to describe and model. Second, to achieve accurate long-term forecasting performance, we design a KalmanNet in a data-driven manner based on the principle of Kalman Filter. Through integrating the residual nonlinear information as control inputs, while treating the biased linear dynamical system as the observation, the KalmanNet predicts and updates to model and refine the uncertainty with Kalman gain. This effectively mitigates the error accumulation of the linear system and helps construct the variational distribution in the space of the measurement function with clear semantics. Compared to diffusion-based models or flow-based models with longer generation processes, which cause more computational consumption and memory overhead, $K^2$VAE adopts a VAE-based structure composed of lightweight but effective KoopmanNet and KalmanNet, which contributes to fast one-step generation and lower memory occupation. The contributions are summarized as follows:

% 1 我们提出了一个高效精准的概率预测生成模型K^2VAE，通过将非线性时间序列转化为隐空间中的线性动力系统，在其中进行精准地预测、修正和不确定性建模，从而构造了具有明确语义的近似后验分布。该方法在短步预测性能优秀的基础上，增强了模型在长步概率预测任务中的表现。

% 2 我们通过利用Koopman Theory充分利用非线性时序数据在测量空间中的潜在的高维线性动力特性，通过数据驱动的方式拟合高维测量函数的映射和线性Koopman算子，将时间序列建模为测量函数空间中的线性动力系统，并且简化了不确定性建模。

% 3 我们应用Kalman Filter整合非线性残差调整koopman Operator生成的有偏线性动力系统，通过逐步计算kalman增益，调整预测值和不确定性分布，可以缓解长步预测任务中不确定性的累积，并在koopman的测量空间中显式建模了不确定性，增强了模型在长步预测方面的能力。

% 4 K^2VAE同时兼顾了短步和长步的概率预测能力，在大量评测中取得SOTA。

\begin{itemize}[left=0.1cm]
    \item To address PTSF, we propose an efficient framework called $K^2$VAE. It transforms nonlinear time series into a linear dynamical system. Through predicting and refining the process uncertainty of the system, $K^2$VAE demonstrates strong generative capability and excells in both the short- and long-term probabilistic forecasting.
    
    \item To distengle the complex nonlinearity in the time series, we design a KoopmanNet to fully exploit the underlying linear dynamical characteristics in the space of measurement function, simplify the modeling, and thus contributing to high model efficiency.
    
    \item To mitigate the error accumulation in LPTSF, we devise a KalmanNet to model, and refine the prediction and uncertainty iteratively.
    
    \item Comprehensive experiments on both short- and long-term PTSF show that $K^2$VAE outperforms state-of-the-art baselines. Additionally, all datasets and code are avaliable at \url{https://github.com/decisionintelligence/K2VAE}.

\end{itemize}

% \item 通过利用库普曼理论，我们充分挖掘了非线性时间序列在测量函数空间中的潜在线性动态特性。

% \item 我们应用卡尔曼滤波器来建模线性系统在测量函数空间中的过程不确定性，从而有效减轻了长期预测中的误差积累。

% \item 在短期和长期概率预测的全面实验中，$K^2$VAE超越了现有的最先进基准方法。此外，所有数据集和代码可以通过 \url{https://github.com/decisionintelligence/K2VAE} 获取。

\section{Preliminaries}
\begin{figure*}[t!]
    \centering
    \includegraphics[width=0.95\linewidth]{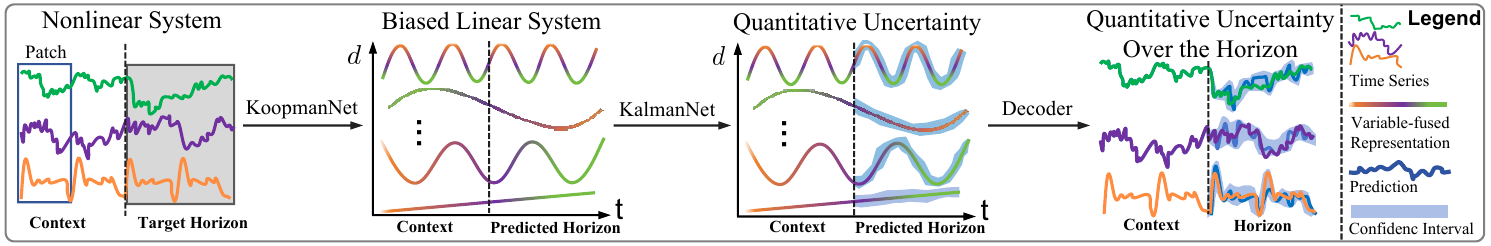}
    \caption{The data flow of $K^2$VAE. It models time series through the KoopmanNet, which constructs a biased linear system. Then the linear system is refined through the KalmanNet while the uncertainty is modeled. Finally, the target distributions over the horzion are predicted through the Decoder. }
    \label{fig:data flow}
\end{figure*}

% \subsection{Koopman Theory}
\textbf{Koopman Theory.} Koopman Theory~\citep{KoopmanTheory, lan2013linearization} is a widely used mathematical tool for dynamic system analysis, providing a way to linearize the nonlinear systems. For nonlinear system $x_{k+1} = f(x_k)$, where $x_k$ denotes system state and $f$ is a nonlinear function, it assumes that the system's state can be mapped into the space of measurement function $\psi$, where it can be modeled by an infinite-dimensional linear Koopman Operator $\mathcal{K}$:
\begin{small}\begin{gather}
    \psi(x_{k+1}) = \psi(f(x_k)) = \mathcal{K} \circ \psi(x_k)
\end{gather}\end{small}%
Koopman Theory helps understand the underlying dynamics of complex nonlinear systems and serves as a powerful tool to linearize them for ease of process.

% \subsection{Kalman Filter}
\textbf{Kalman Filter.} Kalman Filter~\citep{KalmanFilter,simon2001kalman} is a recursive algorithm used for estimating the state of a linear dynamic system. It works in two steps: first, it predicts the current state $x_k$ and uncertainty covariance matrix $\mathrm{P}_k$ based on the system's state transition equation; then, it updates the estimation by incorporating the difference between the measurement and prediction, known as Kalman gain $K_k$. The Kalman Filter effectively fuses information from multiple sensors to enhance estimation accuracy while modeling the uncertainty of the system.

% \subsection{VAE for Probabilistic Time Series Forecasting}
\textbf{VAE for Probabilistic Time Series Forecasting.} PTSF can be treated as a conditional generative task, i.e., generating forecasting horizon \(\hat{Y}=\left[\hat{x}_{T+1}, \hat{x}_{T+2}, \cdots, \hat{x}_{L}\right] \in \mathbb{R}^{N \times L} \) given context series \(X=\left[x_1,x_2,\cdots,x_T\right]\in \mathbb{R}^{N \times T}\), where $N$ denotes the number of variables, $T$ denotes the context length, and $L$ denotes the forecasting horizon. The objective is to model the conditional distribution \(\mathbb{P}(Y|X)\) and sample from it to obtain \(\hat{Y}\). When using Variational AutoEncoder~\citep{bVAE,pu2016variational}, the log-likehood objective is optimized through the Evidence Lower Bound~(\ref{elbo}) which is obtained by Jensen Inequality:
\begin{align}
&\mathcal{L}_{ELBO} = \notag \\
 &-\mathbb{E}[\log \mathbb{P}(Y|Z,X)] + D_{KL}(\mathbb{Q}(Z|X)||\mathbb{P}(Z|X)) \label{elbo}
\end{align}%
In our proposed \(K^2\)VAE, we meticulously construct the variational distribution \(\mathbb{Q}(Z|X)\), aligning it with the uncertainty of the dynamical system. This endows the latent space in \(K^2\)VAE with clear semantics, enhancing its generative capabilities in PTSF.
\section{Methodology}

\subsection{$K^2$VAE Architecture}
As demonstrated in Figure~\ref{fig: overview}, $K^2$VAE consists of four main components: Input Token Embedding, KoopmanNet, KalmanNet, and Decoder. The KoopmanNet and KalmanNet consitute the Encoder of $K^2$VAE. To facilitate comprehension, we present the Data Flow--see Figure~\ref{fig:data flow}. 

Overall, $K^2$VAE employs a meticulously designed pipeline to model the time series at the perspective of dynamic system. First, the Input Token Embedding module patchifys the time series into tokens. Then the KoopmanNet projects them into the space of measurement function, where the inherent nonlinearity and intricate joint distribution between variables are reconsidered for ease. Sequentially, the Koopman Operator is fitted and iterates over the first token to delineate a linear system. Obviously, the perfect measurement function which constructs an absolute linear system is the ideal objective, which means the series generated by the Koopman Operator is biased. We then design the KalmanNet to refine such biased linear system and model the uncertainty by outputting the covariance matrix of multi-dimensional state vector, which assigns the variational posterior distribution $\mathbb{Q}(Z|X)$ in the space of measurement function with clear semantics. 
\begin{figure*}[!htbp]
  \centering
  \includegraphics[width=0.95\linewidth]{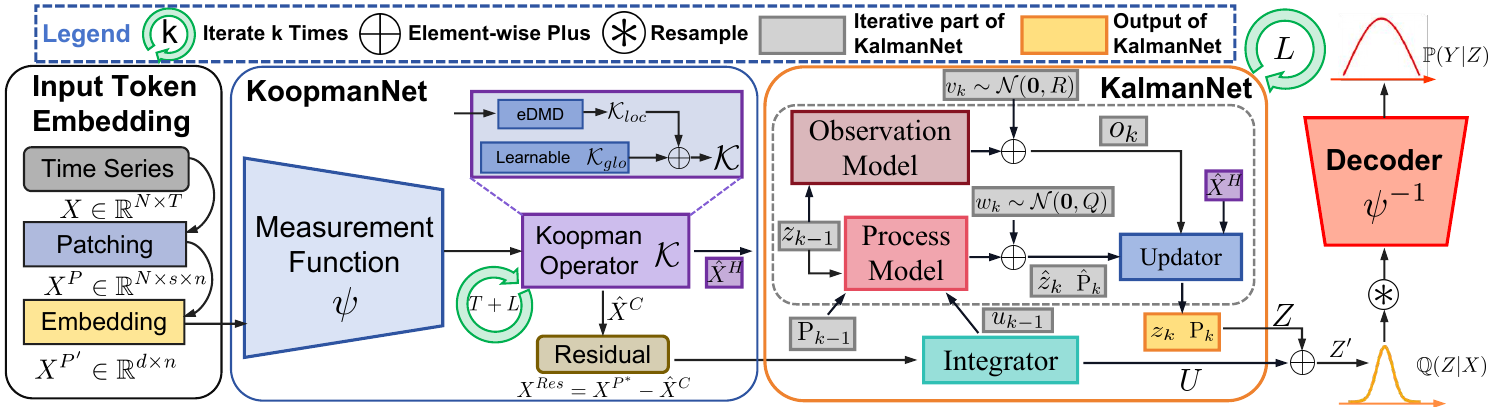}
  \caption{The architecture of $K^2$VAE. Input Token Embedding Module patchifys the time series into tokens and applies Embedding. Encoder Module consists of KoopmanNet and KalmanNet, transforming the tokens into linear system in the space of measurement function, refining it and modeling the process uncertainty as the variational distribution. After resampling from the variational distribution, Decoder module constructs the likehood distribution about the forecasting horizon.}
  \label{fig: overview}
\end{figure*}
The Decoder works as the inverse measurement function $\psi^{-1}$ to map the samples from $\mathbb{Q}(Z|X)$ to the original space, which also serves as the decoder of VAE and models the target distribution $\mathbb{P}(Y|Z,X)$ of the forecasting horizon.

\subsubsection{Input Token Embedding} 
Since Triformer~\cite{Triformer} first proposes the Patching technique, existing works~\citep{ nie2022time, wu2024catch} demonstrate that considering a patch as the ``token'' retains most semantic information and helps establish meaningful state transition procedure for autoregressive models. Our proposed $K^2$VAE also works like an autoregressive dynamic system to model the state transition procedure. Different from those Channel-Independent models which divides patches for each channel and projects them independently, we consider multivariate patches as tokens to implicitly model the cross-variable interaction during state transition. We divide the context series $X =\left[x_1,x_2,\cdots,x_T\right] \in \mathbb{R}^{N \times T}$ into non-overlapping patches:
\begin{gather}
    X^P = \left[x_1^P,x_2^P,\cdots,x_n^P\right] \in \mathbb{R}^{N \times s \times n},
\end{gather}%
where $s=T/n$ denotes the patch size, $n$ denotes the patch number, and $x_i^P \in \mathbb{R}^{N \times s}$ denotes a patch. Then $X^P$ are embeded into high-dimensional hidden space:
\begin{gather}
    X^{P^\prime} = \text{Projection}(\text{Flatten}(X^P)),
 \end{gather}%
where patches are first flattened into $\mathbb{R}^{(N \times s) \times n}$ and then mapped into embeddings $X^{P^\prime} \in \mathbb{R}^{d \times n}$ through a linear projection to fuse the variable information.

\subsubsection{$K^2$VAE Encoder}
\textbf{Linearizing with the KoopmanNet.} Since there exists variable-wise periodic misalignment or temporal non-stationarity in realistic multivariate time series, yielding non-linearity, $K^2$VAE applies Koopman Theory~\citep{KoopmanTheory} to construct the measurement function to project the system states into measurements which can be modeled as a linear system. Practically, we use a learnable MLP-based network to serve as the measurement function $\psi$:
\begin{gather}
    X^{P^\ast} = \psi(X^{P^\prime}) = \left[x^{P^\ast}_1,x^{P^\ast}_2,\cdots,x^{P^\ast}_n\right],
 \end{gather}%
where $X^{P^\ast} \in \mathbb{R}^{d \times n}$ denotes the projected tokens in the measurement space. To capture the transition rule, we utilize the one-step eDMD~\citep{eDMD,koopa} over $X^{P^\ast}$ to efficiently find the best fitted $\mathcal{K}_{loc}$:
\begin{gather}
    X^{P^\ast}_{back} = \left[x^{P^\ast}_1,x^{P^\ast}_2,\cdots,x^{P^\ast}_{n-1}\right],\\
    X^{P^\ast}_{fore} = \left[x^{P^\ast}_2,x^{P^\ast}_3,\cdots,x^{P^\ast}_{n}\right],\\
    \mathcal{K}_{loc} = X^{P^\ast}_{fore} (X^{P^\ast}_{back})^{\dagger},
\end{gather}%
where $(X^{P^\ast}_{back})^{\dagger}$ denotes the Moore-Penrose inverse of $X^{P^\ast}_{back}$. $\mathcal{K}_{loc}$ effectively captures the local transition rule in the space of current measurement function. However, when $\psi$ is underfitted, the low quality of the space may cause numerical instability or guide the model to converge in a wrong direction. To mitigate this issue and capture the global-shared dynamics, we introduce a learnable part $\mathcal{K}_{glo}$. Then we delineate the system through the Koopman Operator $\mathcal{K} = \mathcal{K}_{loc} + \mathcal{K}_{glo} $:
\begin{gather}
    \hat{X}^C = \left[\hat{x}^C_1,\hat{x}^C_2,\cdots,\hat{x}^C_n\right],\\
    \hat{X}^H = \left[\hat{x}^H_1,\hat{x}^H_2,\cdots,\hat{x}^H_m\right],\\
    \hat{x}^C_i = (\mathcal{K})^{i-1} x^{P^\ast}_1 ,\hat{x}^H_i = (\mathcal{K})^{i+n-1} x^{P^\ast}_1, 
 \end{gather}%
where $\hat{X}^C \in \mathbb{R}^{d \times n}$ denotes the reconstruction context generated by Koopman Operator $\mathcal{K}\in \mathbb{R}^{d\times d}$ and $\hat{X}^H \in \mathbb{R}^{d \times m}$ is the predicted horizon, $m = L / s$ means that predicting $L$ steps in the original space is equivalent to predicting $m$ steps in the space of measurement function.

\textbf{Modeling the Uncertainty with the KalmanNet.} Since we adopt a data-driven paradigm to model the measurement function $\psi$ and Koopman Operator $\mathcal{K}$, it exists bias between the generated $\hat{X}^C$ and $X^{P^\ast}$ during optimization, known as a biased linear system. Inspired by Kalman Filter~\citep{KalmanFilter,simon2001kalman} which is born to refine such biased linear sytem, we devise a KalmanNet to model and refine the uncertainty adaptively, aligning it with the variational distribution $\mathbb{Q}(Z|X)$ in the latent measurement space. Specifically, we first fully reuse the nonlinear residual through the Integrator based on an Encoder-Only Vanilla Transformer~\citep{vaswani2017attention}:
\begin{gather}
    X^{Res} = X^{P^\ast} - \hat{X}^C,\\
    U = \text{Integrator}(X^{Res}) = \left[u_1, u_2,\cdots, u_m\right],
 \end{gather}%
where $U \in \mathbb{R}^{d\times m}$ denotes the output integrated by the Integrator. We then construct the Process Model of KalmanNet, which describes the state transition process:
\begin{gather}
    z_k = Az_{k-1} + Bu_k + w_k,\\
    z_0 = x_n^{P^\ast},
 \end{gather}%
where $A\in \mathbb{R}^{d\times d}$ is the state transition matrix, $B \in \mathbb{R}^{d\times d}$ is the control input matrix, and $w_k \sim \mathcal{N}(\mathbf{0},Q)$ is the process noise and $Q$ is its covariance matrix. Sequentially, we construct the Observation Model:
\begin{gather}
    o_k = Hz_k + v_k,
 \end{gather}%
where $H \in \mathbb{R}^{d\times d}$ is the observation matrix and we treat the prediction $\hat{X}^H$ as the prior observation in Update Step~(\ref{obs}). $v_k\sim \mathcal{N}(\mathbf{0},R)$ is the observation noise and $R$ is its covariance matrix. Our goal is to reuse the information from the nonlinear residual, and integrate it into the linear system constructed by KoopmanNet, thus obtaining a more accurate linear system and modeling the uncertainty. In the KalmanNet, all the matrices are learnable. Additionally, we initialize the covariance matrices $Q$ and $R$ as identity matrices and use lower triangular matrices $L_Q$ and $L_R$ to keep the positive definiteness: $Q = L_Q L_Q^T$ and $R = L_R L_R^T$.

Then we conduct the Prediction Step and Update Step iteratively, the Prediction Step can be formulated as:
\begin{gather}
    \hat{z}_{k} = Az_{k-1} + Bu_k,\\
    \hat{\mathrm{P}}_k = A\mathrm{P}_{k-1}A^T + Q,
 \end{gather}%
where $\hat{z}_{k}$ is the predicted state and $\hat{\mathrm{P}}_k$ is the predicted covariance matrix of the process uncertainty. Then the Update Step measures the weight between observation and prediction through Kalman gain $K_k$ to refine the system:
\begin{gather}
    K_k = \hat{\mathrm{P}}_k H^T(H\hat{\mathrm{P}}_k H^T + R)^{-1},\\
    z_k =  \hat{z}_{k} + K_k(\hat{x}_k^H - H \hat{z}_{k}),\label{obs}\\
    \mathrm{P}_k = (I - K_k H) \hat{\mathrm{P}}_k,\label{cov update}
 \end{gather}%
where $z_k$ and $\mathrm{P}_k$ is the refined state vector and covariance matrix. We then obtain the refined predictions $Z=\left[z_1,z_2,\cdots, z_m \right]$ and covariance matrices of each token $\mathrm{P} = \left[\mathrm{P}_1,\mathrm{P}_2,\cdots,\mathrm{P}_m \right]$, which describes the temporal process uncertainty in the dynamical system. We show that the process also obeys the basic assumptions of Koopman Theory in Section~\ref{sec: theory}. To fully utilize the ability of the Integrator, we make a skip connection:
\begin{gather}
    Z^\prime = Z + U
 \end{gather}%
During the training process, the model leverages the Integrator to integrate nonlinear information and gradually adjust the topological structure of the measurement space. Optimized by $\mathcal{L}_{Rec}$~(\ref{recloss}), the deviation of the linear system constructed by the KoopmanNet gradually decreases, causing $ U \to \mathbf{0}$. This facilitates a linear dynamical system in the measurement space and gradually reduces dependence on the Integrator.

\subsubsection{$K^2$VAE Decoder}
After obtaining the prediction $Z^\prime$, and the covariance matrix $\mathrm{P}$ of process uncertainty, the variational distribution is formulated as $\mathbb{Q}(Z|X) = \mathcal{N}(Z^\prime, \mathrm{P})$. During training, we conduct reparameterization sampling from it to keep the ensure the propagation of the gradient. Finally, we utilize the Decoder to map the samples back to the original space and model the $\mathbb{P}(Y|Z)$ with an isotropic Gaussian distribution. Specifically, the Decoder consists of two same MLP structures as the inverse of the Koopman Encoder $\psi$, we formalize them as $\psi^{-1}_\mu$ and $\psi^{-1}_\sigma$:
\begin{gather}
    Z^{sample} = \text{Resample}(\mathbb{Q}(Z|X)),\\
    \mu = \psi^{-1}_\mu(Z^{sample}), \sigma = \psi^{-1}_\sigma(Z^{sample}), \\
    X^{Rec} = \psi^{-1}_\mu(\hat{X}^C),
 \end{gather}%
so that the $\mathbb{P}(Y|Z)=\mathcal{N}(\mu,\sigma)$ is modeled. We also map back the Koopman reconstruction  $\hat{X}^C$ from the measurement space to optimize the $\mathcal{L}_{Rec}$~(\ref{recloss}), which helps measurement function $\psi$ to build a linear system.

\subsubsection{Overall Learning Objective}
The overall learning objective is weightsumed by $\mathcal{L}_{ELBO}$ and $\mathcal{L}_{Rec}$:
\begin{align}
        \mathcal{L}_{ELBO} &= -\mathbb{E}[\log \mathbb{P}(Y|Z,X)] +\notag \\& \ \ \ \ \ \ \ \ \ \ \ \ \ \ \ \ \ \  \ \ D_{KL}(\mathbb{Q}(Z|X)||\mathbb{P}(Z|X)), \\
    \mathcal{L}_{Rec} &= ||X - X^{Rec} ||_2^2,\label{recloss}
 \end{align}%
where the $\mathcal{L}_{ELBO}$ ensures the fundmental mechanism of $K^2$VAE. The prior distribution is $\mathbb{P}(Z|X)=\mathcal{N}(\textbf{0},I)$, where we hope the linear system in measurement space converge to a stable state. $\mathcal{L}_{Rec}$ facilitates the linearization of the measurement space.

\subsection{Theoretical Analysis}
\label{sec: theory}
\subsubsection{The Stability of KalmanNet}
Since the proposed KalmanNet works in a data-driven manner, the floating-point operation error may cause the covariance matrix $\mathrm{P}$ losing positive definiteness, which often occurs in the step~(\ref{cov update}). To mitigate this, we utilize a numerically stable form for this step.
\begin{theorem}
\label{the: stability}
    The positive-definiteness of covariance matrix $P_k$ during the update step $\mathrm{P}_k = (I - K_kH_k)\hat{\mathrm{P}}_k$ can be retained through a numerically stable form:
 \begin{gather}
    \mathrm{P}_k = \frac{1}{2}(\mathrm{P}_k + \mathrm{P}_k^T),\label{1symmetry}\\
     \mathrm{P}_k^{dual} = (I - K_k H_k)\hat{\mathrm{P}}_k(I -K_k H_k)^T + K_k R_k K_k^T,\label{1postive-definite}
  \end{gather}%

\end{theorem}
where (\ref{1symmetry}) ensures the symmetry, (\ref{1postive-definite}) stabilizes the positive-definiteness by decomposing the formula into the sum of two positive definite terms, which better ensures positive definiteness during floating operation. 
\subsubsection{The Convergence of $K^2$VAE}
Since $K^2$VAE models a linear dynamical system in the measurement space where the Koopman Operator serves as the state transition equation, we hope that the convergence state of the KalmanNet does not violate the assumptions of Koopman Theory. In $K^2$VAE, we meticulously design the KalmanNet by making it gradually converge to the Koopman Operator in the forecasting horizon.
\begin{theorem}
\label{the: convergence}
 When $U \to \mathbf{0}$, the state transition equation of the KalmanNet in $K^2$VAE gradually converges to the Koopman Operator.
\end{theorem}
We provide the proof of Theorem~\ref{the: stability}--\ref{the: convergence} in Appendix~\ref{app: theory}.
% \begin{proof}
% Under the assumptions of Koopman Theory, $u_k\to \mathbf{0}$ means the linear system constructed by Koopman Operator has little bias in the current measurement space, which leads to high performance in prediction. Meanwhile, the Predict and Update Steps of $z_t$ are converted to:
% \begin{align}
%     \textit{Predict:} \ \ \ \hat{z}_k &= Az_{k-1}\\
%     \textit{Update:} \ \ \ K_k &= \hat{\mathrm{P}}_k H^T(H\hat{\mathrm{P}}_k H^T + R)^{-1},\\
%     z_k &=  \hat{z}_{k} + K_k(\hat{x}_k^H - H \hat{z}_{k})
%  \end{align}%
% In this basic case, the state transition equation obeys the basic assumptions of Koopman Theory and $A$ can be treated as a ``fine-tuned'' Koopman Operator $\mathcal{K}$ which is enhanced by the Kalman gain and has stronger generalization ability.

% We then consider the special case that KalmanNet fully relys on the observation $\hat{x}_k^H$ from the linear system constructed by Koopman Operator $\mathcal{K}$, thus $H \to I, A \to \mathbf{0}, R \to \mathbf{0}$, the Predict and Update Steps are converted to:
% \begin{align}
%      \textit{Predict:} \ \ \ \hat{z}_k &= \mathbf{0}\\
%     \textit{Update:} \ \ \ z_k &= \hat{x}_k^H
%  \end{align}%
% The system constructed by KalmanNet can be treated as $z_t = \mathcal{K}z_{t-1}$ equivalent to the original Koopman Operator.
% \end{proof}

\begin{table*}[t]
\centering
\caption{Statistical information of the datasets.}
\resizebox{0.9\textwidth}{!}{
\begin{tabular}{l|c|cccrl}
\toprule
 \textbf{Horizon} & \textbf{Dataset} & \textbf{\#var.} & \textbf{range} & \textbf{freq.} & \textbf{timesteps} & \textbf{Description}  \\
 \midrule
 \multirow{7}{*}{\textbf{Long-term}} &  ETTh1/h2-L  & 7 & $\mathbb{R}^+$ & H & 17,420 & Electricity transformer temperature per hour \\
 &  ETTm1/m2-L    & 7 & $\mathbb{R}^+$ & 15min & 69,680 & Electricity transformer temperature every 15 min  \\
 &  Electricity-L & 321 & $\mathbb{R}^+$ & H & 26,304 & Electricity consumption (Kwh) \\
 &  Traffic-L     & 862 & (0,1) & H & 17,544 & Road occupancy rates \\
 &  Exchange-L    & 8 & $\mathbb{R}^+$ & Busi. Day & 7,588 & Daily exchange rates of 8 countries \\
 &  ILI-L         & 7 & (0,1) & W & 966 & Ratio of patients seen with influenza-like illness \\
 &  Weather-L     & 21 & $\mathbb{R}^+$ & 10min & 52,696 & Local climatological data \\
 \midrule
 \multirow{6}{*}{\textbf{Short-term}}   & ETTh1/h2-S & 7 & $\mathbb{R}^+$ & H & 17,420 & Electricity transformer temperature per hour \\
 &  ETTm1/m2-S    & 7 & $\mathbb{R}^+$ & 15min & 69,680 & Electricity transformer temperature every 15 min  \\
 &  Exchange-S & 8   & $\mathbb{R}^+$ & Busi. Day & 6,071 & Daily exchange rates of 8 countries \\
  &  Solar-S    & 137 & $\mathbb{R}^+$ & H & 7,009 & Solar power production records  \\
  &  Electricity-S    & 370   & $\mathbb{R}^+$ & H & 5,833 & Electricity consumption \\
  &  Traffic-S    & 963   & (0,1) & H & 4,001 & Road occupancy rates  \\

\bottomrule
\end{tabular}}
\label{1 Dataset Summary}
\end{table*}

\begin{table*}[!htbp]
\caption{Comparison on short-term probabilistic forecasting scenarios across eight real-world datasets. Lower CRPS or
NMAE values indicate better predictions. The means and standard errors are based on 5 independent runs of retraining and evaluation. \textcolor{red}{\textbf{Red}}: the best, \textcolor{blue}{\underline{Blue}}: the 2nd best.}
    \centering
    \resizebox{0.97\textwidth}{!}{
    \begin{tabular}{c|c|ccccccccc}
    \toprule
        Model & Metric & Exchange-S & Solar-S & Electricity-S & Traffic-S & ETTh1-S & ETTh2-S & ETTm1-S & ETTm2-S  \\ \midrule
        \multirow{2}{*}{FITS} & CRPS & $0.012\scriptstyle{\pm 0.002}$ & $0.516\scriptstyle{\pm 0.011} $& $0.068\scriptstyle{\pm 0.003}$ & $0.298\scriptstyle{\pm 0.022}$ & $0.320\scriptstyle{\pm 0.017}$ & $0.212\scriptstyle{\pm 0.012}$ & $0.193\scriptstyle{\pm 0.005}$ & $0.199\scriptstyle{\pm 0.003}$  \\
        ~ & NMAE & $0.017\scriptstyle{\pm 0.003}$ & $0.701\scriptstyle{\pm 0.014}$ & $0.092\scriptstyle{\pm 0.004}$ & $0.392\scriptstyle{\pm 0.028}$ & $0.423\scriptstyle{\pm 0.033}$ & $0.278\scriptstyle{\pm 0.009}$ & $0.249\scriptstyle{\pm 0.007}$ & $0.260\scriptstyle{\pm 0.011}$ \\ \hline
        \multirow{2}{*}{PatchTST} & CRPS & $0.052\scriptstyle{\pm 0.016}$ & $0.491\scriptstyle{\pm 0.008}$ & $0.063\scriptstyle{\pm 0.003}$ & $0.278\scriptstyle{\pm 0.018}$ & $0.314\scriptstyle{\pm 0.022}$ & $0.207\scriptstyle{\pm 0.006}$ & $0.234\scriptstyle{\pm 0.011}$ & $0.212\scriptstyle{\pm 0.018}$  \\
        ~ & NMAE & $0.069\scriptstyle{\pm 0.013}$ & $0.663\scriptstyle{\pm 0.010}$ & $0.085\scriptstyle{\pm 0.006}$ & $0.363\scriptstyle{\pm 0.023}$ & $0.407\scriptstyle{\pm 0.030}$ & $0.260\scriptstyle{\pm 0.009}$ & $0.271\scriptstyle{\pm 0.009}$ & $0.257\scriptstyle{\pm 0.011}$  \\ \hline
        \multirow{2}{*}{iTransformer} & CRPS & $0.059\scriptstyle{\pm 0.018}$ & $0.504\scriptstyle{\pm 0.012}$ & $0.066\scriptstyle{\pm 0.004}$ & $0.244\scriptstyle{\pm 0.011}$ & $0.317\scriptstyle{\pm 0.020}$ & $0.219\scriptstyle{\pm 0.008}$ & $0.254\scriptstyle{\pm 0.012}$ & $0.201\scriptstyle{\pm 0.018}$  \\
        ~ & NMAE & $0.081\scriptstyle{\pm 0.022}$ & $0.695\scriptstyle{\pm 0.017}$ & $0.087\scriptstyle{\pm 0.006}$ & $0.319\scriptstyle{\pm 0.019}$ & $0.408\scriptstyle{\pm 0.028}$ & $0.276\scriptstyle{\pm 0.017}$ & $0.291\scriptstyle{\pm 0.017}$ & $0.242\scriptstyle{\pm 0.009}$ 
        \\ \hline
        \multirow{2}{*}{Koopa} & CRPS & $0.012\scriptstyle{\pm 0.001}$ & $0.545\scriptstyle{\pm 0.016}$ & $0.085\scriptstyle{\pm 0.014}$ & $0.253\scriptstyle{\pm 0.018}$ & $0.326\scriptstyle{\pm 0.013}$ & $0.211\scriptstyle{\pm 0.019}$ & $0.288\scriptstyle{\pm 0.022}$ & $0.220\scriptstyle{\pm 0.015}$  \\
        ~ & NMAE & $0.015\scriptstyle{\pm 0.002}$ & $0.742\scriptstyle{\pm 0.022}$ & $0.112\scriptstyle{\pm 0.019}$ & $0.330\scriptstyle{\pm 0.019}$ & $0.423\scriptstyle{\pm 0.017}$ & $0.266\scriptstyle{\pm 0.022}$ & $0.329\scriptstyle{\pm 0.026}$ & $0.278\scriptstyle{\pm 0.022}$  \\ \hline
        \multirow{2}{*}{TSDiff} & CRPS & $0.077\scriptstyle{\pm 0.019}$ & $0.568\scriptstyle{\pm 0.015}$ & $0.111\scriptstyle{\pm 0.013}$ & $0.189\scriptstyle{\pm 0.009}$ & $0.304\scriptstyle{\pm 0.016}$ & $0.204\scriptstyle{\pm 0.006}$ &$0.209\scriptstyle{\pm 0.013}$ & $\textcolor{blue}{\underline{0.124}\scriptstyle{\pm 0.008}}$  \\ 
        ~ & NMAE & $0.096\scriptstyle{\pm 0.024}$ & $0.635\scriptstyle{\pm 0.012}$ & $0.115\scriptstyle{\pm 0.018}$ & $0.206\scriptstyle{\pm 0.011}$ & $0.400\scriptstyle{\pm 0.025}$ & $0.272\scriptstyle{\pm 0.015}$ & $0.276\scriptstyle{\pm 0.008}$ & $\textcolor{blue}{\underline{0.162}\scriptstyle{\pm 0.008}}$  \\ \hline
        \multirow{2}{*}{$D^3$VAE} & CRPS & $0.011\scriptstyle{\pm 0.002}$ & $0.769\scriptstyle{\pm 0.029}$ & $0.071\scriptstyle{\pm 0.009}$ & $0.143\scriptstyle{\pm 0.008}$ & $0.324\scriptstyle{\pm 0.019}$ & $0.216\scriptstyle{\pm 0.015}$ &$0.198\scriptstyle{\pm 0.015}$ & $0.303\scriptstyle{\pm 0.024}$  \\ 
        ~ & NMAE & $0.012\scriptstyle{\pm 0.002}$ & $0.998\scriptstyle{\pm 0.049}$ & $0.092\scriptstyle{\pm 0.013}$ & $0.178\scriptstyle{\pm 0.013}$ & $0.410\scriptstyle{\pm 0.016}$ & $0.267\scriptstyle{\pm 0.018}$ & $0.250\scriptstyle{\pm 0.018}$ & $0.378\scriptstyle{\pm 0.031}$  \\ \hline
        \multirow{2}{*}{GRU NVP} & CRPS & $0.019\scriptstyle{\pm 0.006}$ & $0.530\scriptstyle{\pm 0.008}$ & $0.062\scriptstyle{\pm 0.003}$ & $0.168\scriptstyle{\pm 0.008}$ & $0.398\scriptstyle{\pm 0.034}$ & $0.309\scriptstyle{\pm 0.023}$ & $0.455\scriptstyle{\pm 0.029}$ & $0.276\scriptstyle{\pm 0.014}$  \\ 
        ~ & NMAE & $0.024\scriptstyle{\pm 0.007}$ & $0.670\scriptstyle{\pm 0.011}$ & $0.081\scriptstyle{\pm 0.006}$ & $0.209\scriptstyle{\pm 0.013}$ & $0.477\scriptstyle{\pm 0.040}$ & $0.375\scriptstyle{\pm 0.024}$ & $0.584\scriptstyle{\pm 0.047}$ & $0.349\scriptstyle{\pm 0.028}$  \\ \hline
        \multirow{2}{*}{GRU MAF} & CRPS & $0.012\scriptstyle{\pm 0.003}$ & $0.486\scriptstyle{\pm 0.007}$ & $0.056\scriptstyle{\pm 0.002}$ & $0.144\scriptstyle{\pm 0.022}$ & $\textcolor{blue}{\underline{0.258}\scriptstyle{\pm 0.013}}$ & $0.160\scriptstyle{\pm 0.008}$ & $0.151\scriptstyle{\pm 0.009}$ & $0.146\scriptstyle{\pm 0.011}$  \\ 
        ~ & NMAE & $0.016\scriptstyle{\pm 0.002}$ & $0.603\scriptstyle{\pm 0.009}$ & $0.073\scriptstyle{\pm 0.004}$ & $0.182\scriptstyle{\pm 0.029}$ & $\textcolor{blue}{\underline{0.326}\scriptstyle{\pm 0.016}}$ & $0.208\scriptstyle{\pm 0.003}$ & $0.198\scriptstyle{\pm 0.004}$ & $0.193\scriptstyle{\pm 0.008}$  \\ \hline
        \multirow{2}{*}{Trans MAF} & CRPS & $0.012\scriptstyle{\pm 0.001} $& $0.442\scriptstyle{\pm 0.011}$ & $0.054\scriptstyle{\pm 0.002}$ & $0.133\scriptstyle{\pm 0.004}$ & $0.309\scriptstyle{\pm 0.009}$ & $0.200\scriptstyle{\pm 0.012}$ & $\textcolor{blue}{\underline{0.139}\scriptstyle{\pm 0.005}}$ & $0.180\scriptstyle{\pm 0.010}$  \\ 
        ~ & NMAE & $0.016\scriptstyle{\pm 0.001}$ & $0.577\scriptstyle{\pm 0.014}$ & $0.071\scriptstyle{\pm 0.003}$ & $0.160\scriptstyle{\pm 0.006}$ & $0.400\scriptstyle{\pm 0.011}$ & $0.256\scriptstyle{\pm 0.009}$ & $\textcolor{blue}{\underline{0.162}\scriptstyle{\pm 0.006}}$ & $0.224\scriptstyle{\pm 0.009}$  \\ \hline
        \multirow{2}{*}{TimeGrad} & CRPS & $\textcolor{blue}{\underline{0.009}\scriptstyle{\pm 0.001}}$& $0.465\scriptstyle{\pm 0.016}$ & $0.057\scriptstyle{\pm 0.002}$ & $\textcolor{blue}{\underline{0.130}\scriptstyle{\pm 0.005}}$ & $0.273\scriptstyle{\pm 0.007}$ &$0.184\scriptstyle{\pm 0.006}$ & $0.186\scriptstyle{\pm 0.003}$ & $0.148\scriptstyle{\pm 0.004}$  \\
        ~ & NMAE & $\textcolor{blue}{\underline{0.012}\scriptstyle{\pm 0.002}}$& $0.609\scriptstyle{\pm 0.015}$ & $0.073\scriptstyle{\pm 0.004}$ & $\textcolor{red}{\textbf{0.155}\scriptstyle{\pm 0.007}}$ & $0.356\scriptstyle{\pm 0.013}$ & $0.224\scriptstyle{\pm 0.014}$ &$ 0.246\scriptstyle{\pm 0.007}$ & $0.189\scriptstyle{\pm 0.006}$  \\ \hline
        \multirow{2}{*}{CSDI} & CRPS & $0.009\scriptstyle{\pm 0.001}$ & $\textcolor{blue}{\underline{0.392}\scriptstyle{\pm 0.006}}$ & $\textcolor{red}{\textbf{0.051}\scriptstyle{\pm 0.001}}$ & $0.147\scriptstyle{\pm 0.014}$ & $0.262\scriptstyle{\pm 0.012}$ & $\textcolor{blue}{\underline{0.133}\scriptstyle{\pm 0.006}}$ & $0.140\scriptstyle{\pm 0.012}$ & $0.144\scriptstyle{\pm 0.018}$  \\ 
        ~ & NMAE & $0.013\scriptstyle{\pm 0.001}$& $\textcolor{blue}{\underline{0.533}\scriptstyle{\pm 0.007}}$& $\textcolor{red}{\textbf{0.066}\scriptstyle{\pm 0.001}}$ & $0.175\scriptstyle{\pm 0.013}$ & $0.339\scriptstyle{\pm 0.009}$ & $\textcolor{blue}{\underline{0.161}\scriptstyle{\pm 0.013}}$ & $0.169\scriptstyle{\pm 0.021}$ & $0.181\scriptstyle{\pm 0.024}$  \\ \hline
        \multirow{2}{*}{$K^2$VAE} & CRPS & $\textcolor{red}{\textbf{0.009}\scriptstyle{\pm 0.001}}$& $\textcolor{red}{\textbf{0.367}\scriptstyle{\pm 0.005}}$ & $\textcolor{blue}{\underline{0.053}\scriptstyle{\pm 0.002}}$ & $\textcolor{red}{\textbf{0.129}\scriptstyle{\pm 0.004}}$ & $\textcolor{red}{\textbf{0.256}\scriptstyle{\pm 0.008}}$ & $\textcolor{red}{\textbf{0.128}\scriptstyle{\pm 0.006}}$ & $\textcolor{red}{\textbf{0.135}\scriptstyle{\pm 0.008}}$ & $\textcolor{red}{\textbf{0.122}\scriptstyle{\pm 0.008}}$  \\ 
        ~ & NMAE & $\textcolor{red}{\textbf{0.009}\scriptstyle{\pm 0.001}}$ & $\textcolor{red}{\textbf{0.480}\scriptstyle{\pm 0.008}}$ & $\textcolor{blue}{\underline{0.068}\scriptstyle{\pm 0.002}}$& $\textcolor{blue}{\underline{0.157}\scriptstyle{\pm 0.007}}$ & $\textcolor{red}{\textbf{0.312}\scriptstyle{\pm 0.008}}$ & $\textcolor{red}{\textbf{0.140}\scriptstyle{\pm 0.007}}$ & $\textcolor{red}{\textbf{0.152}\scriptstyle{\pm 0.007}}$ & $\textcolor{red}{\textbf{0.146}\scriptstyle{\pm 0.009}}$ \\\bottomrule
    \end{tabular}}
    \label{tab: short-term}
\end{table*}

\section{Experiments}
In this section, we provide empirical results to show the strong performance of $K^2$VAE against state-of-art baselines on both short- and long-term probabilistic forecasting tasks. We also analyze the model efficiency and the key parameters of $K^2$VAE as the proof of architectural superiority.

\begin{table*}[!htbp]
\caption{Comparison on long-term probabilistic forecasting (forecasting horizon L=720) scenarios across nine real-world datasets. Lower CRPS or NMAE values indicate better predictions. The means and standard errors are based on 5 independent runs of retraining and evaluation. \textcolor{red}{\textbf{Red}}: the best, \textcolor{blue}{\underline{Blue}}: the 2nd best. The full results of all four horizons {96, 192, 336, 720} are listed in Table~\ref{tab:long_term_fore_CRPS}, \ref{tab:long_term_fore_NMAE} in Appendix~\ref{app: full results}.}
    \centering
    \resizebox{0.97\textwidth}{!}{
    \begin{tabular}{c|c|cccccccccc}
    \toprule
        Model & Metric & ETTm1-L & ETTm2-L & ETTh1-L & ETTh2-L & Electricity-L & Traffic-L & Weather-L & Exchange-L & ILI-L \\ \midrule
        \multirow{2}{*}{FITS} & CRPS &$0.305\scriptstyle\pm0.024$
        &$0.449\scriptstyle\pm0.034$	&$0.348\scriptstyle\pm0.025$	&$0.314\scriptstyle\pm0.022$	&$0.115\scriptstyle\pm0.024$	&$0.374\scriptstyle\pm0.004$	&$0.267\scriptstyle\pm0.003$	& $\textcolor{blue}{\underline{0.074}\scriptstyle\pm0.011}$
        &$0.211 \scriptstyle{\pm 0.011}$ \\
        ~ & NMAE &$0.406\scriptstyle\pm0.072$	&$0.540\scriptstyle\pm0.052$	&$0.468\scriptstyle\pm0.012$	&$0.401\scriptstyle\pm0.022$	&$0.149\scriptstyle\pm0.012$	&$0.453\scriptstyle\pm0.022$	&$0.317\scriptstyle\pm0.021$	
        &$\textcolor{blue}{\underline{0.097}\scriptstyle\pm0.011}$&$0.245\scriptstyle\pm0.017$  \\ \hline
        \multirow{2}{*}{PatchTST} & CRPS &$0.304\scriptstyle\pm0.029$	&$\textcolor{blue}{\underline{0.229}\scriptstyle\pm0.036}$	&$0.323\scriptstyle\pm0.020$	&$0.304\scriptstyle\pm0.018$	&$0.127\scriptstyle\pm0.015$	&$0.214\scriptstyle\pm0.001$	&$0.142\scriptstyle\pm0.005$	&$0.097\scriptstyle\pm0.007$	
        &$0.233 \scriptstyle{\pm 0.019}$ \\
        ~ & NMAE & $0.382\scriptstyle\pm0.066$	&$\textcolor{blue}{\underline{0.288}\scriptstyle\pm0.034}$	&$0.428\scriptstyle\pm0.024$	&$0.371\scriptstyle\pm0.021$	&$0.164\scriptstyle\pm0.024$	&$\textcolor{blue}{\underline{0.253}\scriptstyle\pm0.012}$	&$0.152\scriptstyle\pm0.029$	&$0.126\scriptstyle\pm0.001$	&$0.287\scriptstyle\pm0.023$ \\ \hline
        \multirow{2}{*}{iTransformer} & CRPS & $0.455\scriptstyle\pm0.021$	&$0.311\scriptstyle\pm0.024$	&$0.350\scriptstyle\pm0.019$	&$0.542\scriptstyle\pm0.015$	&$0.109\scriptstyle\pm0.044$	&$0.284\scriptstyle\pm0.004$	&$0.133\scriptstyle\pm0.004$	&$0.087\scriptstyle\pm0.023$	
        &$0.222 \scriptstyle{\pm 0.020}$ \\
        ~ & NMAE &$0.490\scriptstyle\pm0.038$	&$0.385\scriptstyle\pm0.042$	&$0.449\scriptstyle\pm0.022$	&$0.667\scriptstyle\pm0.012$	&$0.140\scriptstyle\pm0.009$	&$0.361\scriptstyle\pm0.030$	&$0.147\scriptstyle\pm0.019$	&$0.113\scriptstyle\pm0.015$	&$0.278\scriptstyle\pm0.017$ \\ \hline
        \multirow{2}{*}{Koopa} & CRPS & $\textcolor{blue}{\underline{0.295}\scriptstyle\pm0.027}$	&$0.233\scriptstyle\pm0.025$	&$\textcolor{blue}{\underline{0.318}\scriptstyle\pm0.009}$	&$\textcolor{blue}{\underline{0.293}\scriptstyle\pm0.026}$	&$0.113\scriptstyle\pm0.018$	&$0.358\scriptstyle\pm0.022$	&$0.140\scriptstyle\pm0.007$	&$0.091\scriptstyle\pm0.012$	
        &$0.228 \scriptstyle{\pm 0.022}$ \\
        ~ & NMAE &$\textcolor{blue}{\underline{0.377}\scriptstyle\pm0.037}$	&$0.290\scriptstyle\pm0.033$	&$\textcolor{blue}{\underline{0.412}\scriptstyle\pm0.008}$	&$\textcolor{blue}{\underline{0.286}\scriptstyle\pm0.042}$	&$0.149\scriptstyle\pm0.025$	&$0.432\scriptstyle\pm0.032$	&$0.162\scriptstyle\pm0.009$	&$0.116\scriptstyle\pm0.022$	&$0.288\scriptstyle\pm0.031$ \\ \hline
        \multirow{2}{*}{TSDiff} & CRPS &$0.478\scriptstyle\pm0.027$	&$0.344\scriptstyle\pm0.046$	&$0.516\scriptstyle\pm0.027$	&$0.406\scriptstyle\pm0.056$	&$0.478\scriptstyle\pm0.005$	&$0.391\scriptstyle\pm0.002$	&$0.152\scriptstyle\pm0.003$	&$0.082\scriptstyle\pm0.010$	
        &$0.263 \scriptstyle{\pm 0.022}$  \\ 
        ~ & NMAE & $0.622\scriptstyle\pm0.045$	&$0.416\scriptstyle\pm0.065$	&$0.657\scriptstyle\pm0.017$	&$0.482\scriptstyle\pm0.022$	&$0.622\scriptstyle\pm0.142$	&$0.478\scriptstyle\pm0.006$	&$0.141\scriptstyle\pm0.026$	&$0.142\scriptstyle\pm0.009$	&$0.272\scriptstyle\pm0.020$ \\ \hline
        \multirow{2}{*}{GRU NVP} & CRPS &$0.546\scriptstyle\pm0.036$	&$0.561\scriptstyle\pm0.273$	&$0.502\scriptstyle\pm0.039$	&$0.539\scriptstyle\pm0.090$	&$0.114\scriptstyle\pm0.013$	&$\textcolor{blue}{\underline{0.211}\scriptstyle\pm0.004}$	&$0.110\scriptstyle\pm0.004$	&$0.079\scriptstyle\pm0.009$	&$0.307\scriptstyle\pm0.005$  \\ 
        ~ & NMAE & $0.707\scriptstyle\pm0.050$	&$0.749\scriptstyle\pm0.385$	&$0.643\scriptstyle\pm0.046$	&$0.688\scriptstyle\pm0.161$	&$0.144\scriptstyle\pm0.017$	&$0.264\scriptstyle\pm0.006$	&$0.135\scriptstyle\pm0.008$	&$0.103\scriptstyle\pm0.009$	&$0.333\scriptstyle\pm0.005$ \\ \hline
        \multirow{2}{*}{GRU MAF} & CRPS &$0.536\scriptstyle\pm0.033$	&$0.272\scriptstyle\pm0.029$	&$0.393\scriptstyle\pm0.043$	&$0.990\scriptstyle\pm0.023$	&$\textcolor{blue}{\underline{0.106}\scriptstyle\pm0.007}$	& -	&$0.122\scriptstyle\pm0.006$	&$0.160\scriptstyle\pm0.019$	
        &$0.172 \scriptstyle{\pm 0.034}$ \\ 
        ~ & NMAE & $0.711\scriptstyle\pm0.081$	&$0.355\scriptstyle\pm0.048$	&$0.496\scriptstyle\pm0.019$	&$1.092\scriptstyle\pm0.019$	&$0.136\scriptstyle\pm0.098$	
        &-	&$0.149\scriptstyle\pm0.034$	&$0.182\scriptstyle\pm0.010$	&$0.216\scriptstyle\pm0.014$  \\ \hline
        \multirow{2}{*}{Trans MAF} & CRPS & $0.688\scriptstyle\pm0.043$	&$0.355\scriptstyle\pm0.043$	&$0.363\scriptstyle\pm0.053$	&$0.327\scriptstyle\pm0.033$	&-	&-	&$0.113\scriptstyle\pm0.004$	&$0.148\scriptstyle\pm0.017$	&$\textcolor{blue}{\underline{0.155} \scriptstyle{\pm 0.018}}$ \\ 
        ~ & NMAE &$0.822\scriptstyle\pm0.034$	&$0.475\scriptstyle\pm0.029$	&$0.455\scriptstyle\pm0.025$	&$0.412\scriptstyle\pm0.020$	&-	&-	&$0.148\scriptstyle\pm0.040$	&$0.191\scriptstyle\pm0.006$	&$\textcolor{blue}{\underline{0.183}\scriptstyle\pm0.019}$ \\ \hline
        \multirow{2}{*}{TimeGrad} & CRPS & $0.621\scriptstyle\pm0.037$	&$0.470\scriptstyle\pm0.054$	&$0.523\scriptstyle\pm0.027$	&$0.445\scriptstyle\pm0.016$	&$0.108\scriptstyle\pm0.003$	&$0.220\scriptstyle\pm0.002$	&$0.113\scriptstyle\pm0.011$	&$0.099\scriptstyle\pm0.015$	&$0.295\scriptstyle\pm0.083$ \\
        ~ & NMAE & $0.793\scriptstyle\pm0.034$	&$0.561\scriptstyle\pm0.044$	&$0.672\scriptstyle\pm0.015$	&$0.550\scriptstyle\pm0.018$	&$\textcolor{blue}{\underline{0.134}\scriptstyle\pm0.004}$	&$0.263\scriptstyle\pm0.001$	&$0.136\scriptstyle\pm0.020$	&$0.113\scriptstyle\pm0.016$	&$0.325\scriptstyle\pm0.068$  \\ \hline
        \multirow{2}{*}{CSDI} & CRPS & $0.448\scriptstyle\pm0.038$	&$0.239\scriptstyle\pm0.035$	&$0.528\scriptstyle\pm0.012$	&$0.302\scriptstyle\pm0.040$	&-	&-	&$\textcolor{blue}{\underline{0.087}\scriptstyle\pm0.003}$	&$0.143\scriptstyle\pm0.020$	&$0.283\scriptstyle\pm0.012$  \\ 
        ~ & NMAE & $0.578\scriptstyle\pm0.051$	&$0.306\scriptstyle\pm0.040$	&$0.657\scriptstyle\pm0.014$	&$0.382\scriptstyle\pm0.030$	&-	&-	&$ \textcolor{blue}{\underline{0.102}\scriptstyle\pm0.005}$	&$0.173\scriptstyle\pm0.020$	&$0.299\scriptstyle\pm0.013$ \\ \hline
        \multirow{2}{*}{$K^2$VAE} & CRPS & $\textcolor{red}{\textbf{0.294}\scriptstyle\pm0.026}$	&$\textcolor{red}{\textbf{0.221}\scriptstyle\pm0.023}$	&$\textcolor{red}{\textbf{0.314}\scriptstyle\pm0.011}$	&$\textcolor{red}{\textbf{0.280}\scriptstyle\pm0.014}$	&$\textcolor{red}{\textbf{0.057}\scriptstyle\pm0.005}$	&$\textcolor{red}{\textbf{0.200}\scriptstyle\pm0.001}$	&$\textcolor{red}{\textbf{0.084}\scriptstyle\pm0.003}$	&$\textcolor{red}{\textbf{0.069}\scriptstyle\pm0.005}$	&$\textcolor{red}{\textbf{0.142} \scriptstyle{\pm 0.008}}$ \\ 
        ~ & NMAE & $\textcolor{red}{\textbf{0.373}\scriptstyle\pm0.032}$	&$\textcolor{red}{\textbf{0.275}\scriptstyle\pm0.035}$	&$\textcolor{red}{\textbf{0.396}\scriptstyle\pm0.012}$	&$\textcolor{red}{\textbf{0.278}\scriptstyle\pm0.020}$	&$\textcolor{red}{\textbf{0.117}\scriptstyle\pm0.019}$	&$\textcolor{red}{\textbf{0.248}\scriptstyle\pm0.010}$	&$\textcolor{red}{\textbf{0.099}\scriptstyle\pm0.009}$	&$\textcolor{red}{\textbf{0.084}\scriptstyle\pm0.017}$	&$\textcolor{red}{\textbf{0.167}\scriptstyle\pm0.007}$\\\bottomrule
        \multicolumn{8}{l}{Due to the excessive time and memory consumption, some results are unavailable and denoted as -.}
    \end{tabular}}
    \label{tab: long-term}
    
\end{table*}

\subsection{Experimental Setup}
\textbf{Datasets.} We conduct experiments on 8 datasets of short-term forecasting and 9 datasets of long-term forecasting based on ProbTS~\citep{ProbTS}, a comprehensive benchmark used to evaluate probabilistic forecasting tasks. Specifically, we use the datasets ETTh1-S, ETTh2-S, ETTm1-S, ETTm2-S, Electricity-S, Solar-S, Traffic-S, and Exchange-S for short-term forecasting, of which the context length $T$ is equivalent to forecasting horizon $L$ with $T=L=30$ for Exchange-S and $T=L=24$ for the others. For long-term forecasting, we use the datasets ETTh1-L, ETTh2-L, ETTm1-L, ETTm2-L, Electricity-L, Traffic-L, Exchange-L, Weather-L, and ILI-L with forecasting horizon $L \in \{24, 36, 48, 60\}$ for ILI-L and $L \in \{96, 192, 336, 720\}$ for the others. Note that we fix the context length of all the models with $T=36$ for ILI-L and $T=96$ for the others to ensure a fair comparison. Details are shown in Table~\ref{1 Dataset Summary}.

\textbf{Baselines.} We compare $K^2$VAE with 11 strong baselines, including 4 point forecasting models: FITS~\citep{xu2024fitsmodelingtimeseries}, PatchTST~\citep{nie2022time}, iTransformer~\citep{liu2023itransformer}, and Koopa~\citep{koopa}, as well as 7 generative models: TSDiff~\citep{TSDiff}, $D^3$VAE~\citep{li2022generative}, GRU NVP, GRU MAF, Trans MAF~\citep{normalizing-flow}, TimeGrad~\citep{TimeGrad}, and CSDI~\citep{CSDI}, in both short-term and long-term probabilistic forecasting scenarios. The point forecasting models are equipped with gaussian heads to predict the distributions. Detailed descriptions of these models can be found in Appendix~\ref{appendix BASELINES}.

\textbf{Evaluation Metrics.} 
We use two commonly-used metrics CPRS (Continuous Ranked Probability Score) and NMAE (Normalized Mean Absolute Error) to evaluate the probabilistic forecasts. Detailed descriptions of these metrics can be found in Appendix~\ref{appendix Metrics}.

\subsection{Main Results}

Comprehensive probabilistic forecasting results are listed in Table~\ref{tab: short-term} and Table~\ref{tab: long-term} with the best in red and the second in blue. We have the following observations: \par
1) $K^2$VAE outperforms state-of-the-art baselines, showing notable improvements in predictive performance. In short-term scenarios, it achieves a \emph{7.3\%} reduction in CRPS and \emph{14.5\%} reduction in NMAE compared to the second-best baseline, CSDI. In long-term scenarios, it surpasses PatchTST with improvements of \emph{20.9\%} and \emph{19.9\%}. \par
2) $K^2$VAE shows significant advantage on nonstationary time seris datasets such as Exchange-S and Exchange-L. There exists distribution drift phenomenon in these datasets, which causes non-linearity and hinders the prediction and uncertainty modeling. Though diffusion-based and flow-based models are theoretically capable of fitting any distributions, they struggle to construct explict probability transition paths to reach such complex destinations. While $K^2$VAE simplifys this difficulty through modeling the time series in a linear dynamical system, where the uncertainty is more explicit and easier to be modeled. \par
3) $K^2$VAE also shows stable and strong performance with respect to the varying forecasting horizons--see Table~\ref{tab:long_term_fore_CRPS} and Table~\ref{tab:long_term_fore_NMAE} in Appendix~\ref{app: full results}. The performance of most baselines drops significantly as the forecasting horizon extends, while $K^2$VAE maintains superiority. One potential reason is that $K^2$VAE utilizes the KoopmanNet to effectively handle the inherent nonlinearity in long-term forecasting. Another reason is that the KalmanNet mitigates the error accumulation by integrating diverse information.

\subsection{Ablation Studies}
% 所有分析在ETTm1-L Electricity-L Exchange-S Solar-S上做，长步放720，剩下放附录
\subsubsection{Variants of Koopman Operator}

%Koopman Encdoer/Decoder的参数分析，隐藏维度，层数
%Koopman 算子 eDMD和learnable part的使用分析
%得出的结论是如果能同时用会更好，有的数据集上无法使用静态，只能用动态。这是因为我们基于数据驱动的做法处理时间序列，可能初始的空间结构不好，此时使用eDMD可能会因为ill-conditioned的矩阵问题而导致无法计算
We design the local Koopman Operator $\mathcal{K}_{loc}$ obtained by one-step eDMD and a learnable part $\mathcal{K}_{glo}$. As one-step eDMD estimates the $\mathcal{K}_{loc}$ through matrix calculation, which relys on the local quality of the space of measurement function. Once the initilization leads to an ill-conditioned topological structure, the one-step eDMD suffers from the numerical calculation error or guides the model to converge in a wrong direction, which often occurs in long-term probabilistic forecasting scenarios and hinders the performance. To enhance the robustness, we adopt a global learnable part $\mathcal{K}_{glo}$ to mitigate this phenomenon while capturing the global-shared dynamics. As shown in Table~\ref{tab: koopman operator}, mixed Koopman Operator $\mathcal{K} = \mathcal{K}_{loc} + \mathcal{K}_{glo}$ demonstrates better performance on both short- and long-term tasks while providing roubustness to avoid calculation error if $\mathcal{K}_{loc}$ fails.  Complete experimental results are provided in Table~\ref{tab: complete koopman operator} in Appendix~\ref{app: model analysis}.

\begin{table}[!htbp]
    \centering
    \caption{Comparison on different Koopman Operators. Lower values indicate better performance. \textcolor{red}{\textbf{Red}}: the best. L: the forecasting horizon. The full results can be found in Table~\ref{tab: complete koopman operator} in Appendix~\ref{app: model analysis}.}
    \resizebox{\columnwidth}{!}{
        \begin{tabular}{c|c|cccc}
    \toprule
       Koopman & \multirow{2}{*}{Metrics} & Electricity-L & ETTm1-L & Exchange-S & Solar-S \\ 
         Operator & &(L = 720) &(L = 720) &(L = 30) &(L = 24)\\\hline
        \multirow{2}{*}{$\mathcal{K}_{loc}$} & CRPS & - & - & $0.012\scriptstyle\pm0.002$ & $0.450\scriptstyle\pm0.012$ \\ 
        ~ & NMAE & - & - & $0.014\scriptstyle\pm0.002$ & $0.566\scriptstyle\pm0.015$ \\ \midrule
        \multirow{2}{*}{$\mathcal{K}_{glo}$} & CRPS & $0.065\scriptstyle\pm0.007$ & $0.311\scriptstyle\pm0.024$ & $0.011\scriptstyle\pm0.001$ & $0.374\scriptstyle\pm0.004$  \\ 
        ~ & NMAE & $0.130\scriptstyle\pm0.024$ & $0.395\scriptstyle\pm0.027$ & $0.013\scriptstyle\pm0.001$ & $0.488\scriptstyle\pm0.008$ \\ \midrule
        \multirow{2}{*}{$\mathcal{K}_{loc}+\mathcal{K}_{glo}$} & CRPS & $\textcolor{red}{\textbf{0.057}\scriptstyle\pm 0.005}$ & $\textcolor{red}{\textbf{0.294}\scriptstyle\pm0.026}$ & $\textcolor{red}{\textbf{0.009}\scriptstyle\pm0.001}$ & $\textcolor{red}{\textbf{0.367}\scriptstyle\pm0.005}$  \\ 
        ~ & NMAE & $\textcolor{red}{\textbf{0.117}\scriptstyle\pm 0.019}$ & $\textcolor{red}{\textbf{0.373}\scriptstyle\pm0.032}$ & $\textcolor{red}{\textbf{0.009}\scriptstyle\pm0.001}$ & $\textcolor{red}{\textbf{0.480}\scriptstyle\pm0.008}$\\
        \bottomrule
        \multicolumn{6}{l}{Due to the numerical instability, some results are unavailable and denoted as -.}
    \end{tabular}
    }
    
    \label{tab: koopman operator}
\end{table}

\subsubsection{Connections in KalmanNet}
We adopt an Integrator to assist the KalmanNet for faster convergence, which potentially helps tune the topological structure of the space of measurement function into the linear dynamical system. Specifically, the Integrator integrates the non-linear residual into the control input of KalmanNet, which is proved not to affect the prior of Koopman Theory in Section~\ref{sec: theory}. We also make a skip connection between Integrator and the KalmanNet for the final prediction, this constraints Integrator predicting residuals from residuals. Since the space of the measurement space is optimized to converge to a linear system, the Integrator serves as an assistant and gradually stop helping the model. We showcase the different variants in Table~\ref{tab: kalman filter}, to which only some of the features mentioned above are applied. 
\begin{table}[!htbp]
    \centering
    \caption{Comparison on different connections in KalmanNet. Lower values indicate better performance. \textcolor{red}{\textbf{Red}}: the best. L: the forecasting horizon. The full results can be found in Table~\ref{tab: complete kalman filter} in Appendix~\ref{app: model analysis}.}
    \resizebox{0.95\columnwidth}{!}{
    \begin{tabular}{c|c|cccc}
    \toprule
        Connections  & \multirow{2}{*}{Metrics} & Electricity-L & ETTm1-L & Exchange-S & Solar-S \\ 
        in KalmanNet & 
        & (L = 720) & (L = 720) &  (L = 30) & (L = 24)\\ \hline
        \multirow{2}{*}{w/o Integrator} & CRPS & $0.082\scriptstyle\pm0.011$& $0.359\scriptstyle\pm0.024$ & $0.015\scriptstyle\pm0.002$ & $0.398\scriptstyle\pm0.005$\\ 
        ~ & NMAE & $0.188\scriptstyle\pm0.028$ & $0.442\scriptstyle\pm0.029$ & $0.022\scriptstyle\pm0.001$ & $0.531\scriptstyle\pm0.010$ \\ \midrule
        \multirow{2}{*}{w/o skip connection} & CRPS & $0.063\scriptstyle\pm0.007$ &$ 0.315\scriptstyle\pm0.016$ & $0.011\scriptstyle\pm0.001$ & $0.388\scriptstyle\pm0.006$ \\ 
        ~ & NMAE & $0.131\scriptstyle\pm0.015$ & $0.402\scriptstyle\pm0.030$ & $0.013\scriptstyle\pm0.004$ & $0.511\scriptstyle\pm0.008$\\ \midrule
        \multirow{2}{*}{w/o control input} & CRPS & $0.069\scriptstyle\pm0.005$ & $0.322\scriptstyle\pm0.017$ & $0.013\scriptstyle\pm0.006$ & $0.423\scriptstyle\pm0.005$ \\ 
        ~ & NMAE & $0.142\scriptstyle\pm0.018$ & $0.418\scriptstyle\pm0.022$ & $0.017\scriptstyle\pm0.003$ & $0.560\scriptstyle\pm0.009$ \\ \midrule
        \multirow{2}{*}{Mixed}& CRPS & $\textcolor{red}{\textbf{0.057}\scriptstyle\pm 0.005}$ & $\textcolor{red}{\textbf{0.294}\scriptstyle\pm0.026}$ & $\textcolor{red}{\textbf{0.009}\scriptstyle\pm0.001}$ & $\textcolor{red}{\textbf{0.367}\scriptstyle\pm0.005}$  \\ 
        ~ & NMAE & $\textcolor{red}{\textbf{0.117}\scriptstyle\pm 0.019}$ & $\textcolor{red}{\textbf{0.373}\scriptstyle\pm0.032}$ & $\textcolor{red}{\textbf{0.009}\scriptstyle\pm0.001}$ & $\textcolor{red}{\textbf{0.480}\scriptstyle\pm0.008}$ \\ \bottomrule
    \end{tabular}
    }
    \label{tab: kalman filter}
\end{table}

We observe that our adopted Mixed variant outperforms others, because it provides gains for the KalmanNet, which integrates non-linear information for adaption, and provides constraints for the Integrator, which makes full use of it without destroying the assumptions of Koopman Theory. The variant ``w/o skip connection'' completely depends on the linear fitting ability of the KalmanNet, which is hard to disentangle the nonlinear components in the early stages of training, thus potentially hindering the modeling of process uncertainty. On the other hand, the variant ``w/o control input'' gives too little support for KalmanNet to adaptively refine the prediction and process uncertainty. Complete experimental results are provided in Table~\ref{tab: complete kalman filter} in Appendix~\ref{app: full results}.

\subsubsection{Ablations of KoopmanNet \& KalmanNet}
As the most important modules, the KoopmanNet and KalmanNet jointly contribute to state-of-the-art performance of $K^2$VAE. To evaluate their impactment, we conduct ablation studies and the results are shown in Table~\ref{tab: koopman kalman}.

It is observed that both the KoopmanNet and KalmanNet show indispensability in probabilistic forecasting. Since the KoopmanNet ensures the linearization of the modeling, it shows greater impact in forecasting performance. Another reason is that KalmanNet does not excell at non-linear modeling because it is based on the linear kalman filter.
\begin{table}[htbp!]
    \centering
    \caption{Ablations on KoopmanNet and KalmanNet. Lower values indicate better performance. \textcolor{red}{\textbf{Red}}: the best. L: the forecasting horizon. The full results can be found in Table~\ref{tab: complete koopman kalman} in Appendix~\ref{app: model analysis}.}
    \resizebox{0.95\columnwidth}{!}{
    \begin{tabular}{c|c|cccc}
    \toprule
        \multirow{2}{*}{Variants}  & \multirow{2}{*}{Metrics} & Electricity-L & ETTm1-L & Exchange-S & Solar-S \\ 
         & 
        & (L = 720) & (L = 720) &  (L = 30) & (L = 24)\\ \hline
        \multirow{2}{*}{w/o KoopmanNet} & CRPS & $0.074\scriptstyle\pm0.009$& $0.443\scriptstyle\pm0.034$ & $0.014\scriptstyle\pm0.002$ & $0.385\scriptstyle\pm0.008$\\ 
        ~ & NMAE & $0.162\scriptstyle\pm0.015$ & $0.601\scriptstyle\pm0.058$ & $0.016\scriptstyle\pm0.001$ & $0.528\scriptstyle\pm0.014$ \\ \midrule
        \multirow{2}{*}{w/o KalmanNet} & CRPS & $0.089\scriptstyle\pm0.011$ &$ 0.398\scriptstyle\pm0.038$ & $0.011\scriptstyle\pm0.001$ & $0.375\scriptstyle\pm0.005$ \\  
        ~ & NMAE & $0.192\scriptstyle\pm0.023$ & $0.539\scriptstyle\pm0.044$ & $0.012\scriptstyle\pm0.001$ & $0.499\scriptstyle\pm0.009$\\ \midrule
        \multirow{2}{*}{$K^2$VAE}& CRPS & $\textcolor{red}{\textbf{0.057}\scriptstyle\pm 0.005}$ & $\textcolor{red}{\textbf{0.294}\scriptstyle\pm0.026}$ & $\textcolor{red}{\textbf{0.009}\scriptstyle\pm0.001}$ & $\textcolor{red}{\textbf{0.367}\scriptstyle\pm0.005}$  \\ 
        ~ & NMAE & $\textcolor{red}{\textbf{0.117}\scriptstyle\pm 0.019}$ & $\textcolor{red}{\textbf{0.373}\scriptstyle\pm0.032}$ & $\textcolor{red}{\textbf{0.009}\scriptstyle\pm0.001}$ & $\textcolor{red}{\textbf{0.480}\scriptstyle\pm0.008}$ \\ \bottomrule
    \end{tabular}
    }

    \label{tab: koopman kalman}
\end{table}
\subsection{Model Efficiency}
\begin{figure}[!htbp]
  \centering
  \includegraphics[width=0.97\linewidth]{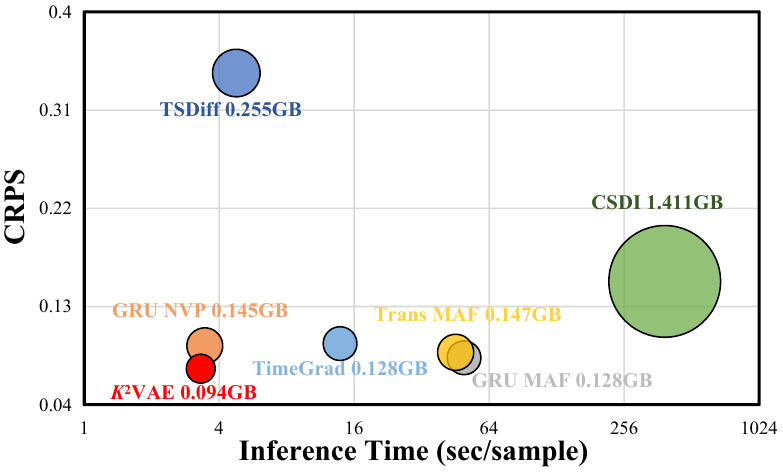}
  \caption{Model efficiency comparison. All the statistical data is obtained on the Electricity-L ($T=L=96$). Sample-wise inference time and max gpu memory is obtained with batch size equals 1. Lower values of CRPS indicate better performance.}
  \label{fig: footprint}
\end{figure}
%训练时间，运行时间，参考koopa，probts
We evaluate the model efficiency from three aspects: probabilistic forecasting performance (CRPS), inference time (sec/sample), and max gpu memory (GB). Figure~\ref{fig: footprint} showcases a common scenario on Electricity-L (96-96), which reflects the overall relative relationships on above-mentioned three aspects. $K^2$VAE achieves best forecasting performance while occupying the minimum gpu memory and having the fastest inference speed. One reason is that $K^2$VAE applies KoopmanNet and KalmanNet, composed of several lightweight MLPs or linear layers, to efficiently build the variational distribution as process uncertainty, thus enhancing generation capability of $K^2$VAE. Another reason is that $K^2$VAE utilizes the VAE architecture and obeys the one-step-generation paradigm, while diffusion-based or flow-based models have longer probabilistic transition paths, which produces more intermediate results and consumes longer duration. More evidence of model efficiency is provided in Table~\ref{tab: complete model efficiency} in Appendix~\ref{app: full results}.

% \subsubsection{Visualization of forecasting}
% % study cases
% \begin{figure}[!htbp]
%   \centering
%     \includegraphics[width=0.9\linewidth]{Figures/image.png}
%   \caption{Take some space}
%   \label{fig: study cases}
% \end{figure}
\section{Related Works}

\subsection{Probabilistic Time Series Forecasting}
% 概率预测旨在提供目标变量的预测分布，而不是像确定性预测那样仅产生单一的点估计。随着深度学习的快速发展，概率预测模型逐渐受到越来越多的关注，并不断涌现出新的方法。DeepAR (Salinas et al., 2020) 采用递归神经网络（RNN）来建模隐藏状态的转换，并通过生成高斯分布来进行预测。作为DeepAR的扩展，GPVar (Salinas et al., 2019) 利用高斯copula将数据进行变换，并假设这些变换后的数据符合多元高斯分布。Rangapuram et al. (2018b)、Salinas et al. (2020) 和 Li et al. (2021) 将状态空间模型与深度学习相结合，以提升预测精度。Feng et al. (2023) 和 Tang & Matteson (2021) 提出了基于注意力机制的方法，增强了模型对长距离依赖关系的捕捉能力，从而进一步提高了预测准确性。扩散模型（如Rasul et al., 2021；Li et al., 2022；Fan et al., 2024）将预测任务视为去噪过程，在高维数据的处理上表现优异。另一种方法则是采用更为复杂的分布形式，如正则化流（normalizing flows）(Rasul et al., 2020)，以进一步改进预测的表现。

Probabilistic forecasting aims to provide the predictive distribution of the target variable. With the rapid development of deep learning, new methods are continually emerging. DeepAR~\cite{DeepAR} uses recurrent neural networks (RNNs) to model the transitions of hidden states and generates a Gaussian distribution for predictions. Following the autoregressive paradigm, DeepState~\cite{DeepState} and DSSMF~\cite{DSSMF} combine state space models with deep learning to improve forecasting accuracy. MANF~\cite{MANF} and ProTran~\cite{ProTran} introduced attention-based methods that enhance the model's ability to capture long-range dependencies, further improving forecasting accuracy. Diffusion models, such as those proposed by TimeGrad~\cite{TimeGrad}, TSDiff~\cite{TSDiff}, and CSDI~\cite{CSDI}, approach the forecasting task as a denoising process, excelling in handling high-dimensional data. Another approach involves using more complex distribution forms, such as normalizing flows~\cite{normalizing-flow}, to further enhance forecasting performance. Compared with RNN-based or State Space models, $K^2$VAE also autoregressively models the time series in a linear dynamical system, but mitigates the error accumulation through KalmanNet. Compared with generative diffusion-based or flow-based models, $K^2$VAE adopts the VAE structure and follows single-step-generation principle, achieves faster inference speed, lower memory occupation, and better performance.

\subsection{VAE for Time Series}
Variational Autoencoders (VAEs)~\cite{kingma2013auto} have found wide applicability across various time series tasks. In time series generation, VAEs synthesize time series by encoding the data into a lower-dimensional latent space and then decoding it to recreate similar sequences, which helps preserve the statistical properties of the original data, making VAEs valuable for data augmentation~\cite{li2023causal,Timevae}. In time series imputation, VAEs recover missing values by learning the underlying latent structure of the data~\cite{boquet2019missing,li2021variational}. By capturing temporal dependencies and relationships, they help restore incomplete time series with high accuracy. In time series anomaly detection, VAEs are used to learn the expected patterns within time series data and flag deviations that indicate anomalous behavior~\cite{huang2022semi,wang2024revisiting}. In time series forecasting, TimeVAE~\citep{Timevae} and $D^3$VAE~\citep{li2022generative} are tailored for short-term probabilistic foercasting tasks. Koopa~\citep{koopa}, as a strong baseline based on Koopman Theory, is tailored for long-term deterministic forecasting by adopting multi-scale MLP structures, which also falls short in probabilistic forecasting. Compared to these methods, $K^2$VAE is tailored for LPTSF, which considers the inherent nonlinearity of time series through a KoopmanNet and tackles the error accumulation through a KalmanNet, thus enhancing the ability to predict long-term future distributions. This facilitates better decision-making in dynamic and uncertain environments.

\section{Conclusion}
In this work, we propose a VAE-based probabilistic forecasting model called $K^2$VAE to solve PTSF. By leveraging the KoopmanNet, $K^2$VAE transforms nonlinear time series into a linear dynamical system, which allows for a more effective representation of state transitions and the inherent process uncertainties. Furthermore, the KalmanNet provides a solution to model the uncertainty in the linear dynamical system, mitigating the error accumuation in long-term forecasting tasks. Through comprehensive experiments, we demonstrate that $K^2$VAE not only outperforms existing state-of-the-art methods in both short- and long-term probabilistic forecasting tasks, but also achieves fascinating model efficiency. 

In the future, we hope to continuously study the one-step generation paradigm in time series probabilistic modeling to further improve the model performance and efficiency. Another primary direction is to pionner the exploration of foundation probabilistic time series forecasting models, which can work effectively in zero-shot scenarios.

\clearpage
\section*{Impact Statement}
This paper presents work whose goal is to advance the field of Machine Learning. There are many potential societal consequences of our work, none which we feel must be specifically highlighted here. 

\section*{Acknowledgements}
This work was partially supported by National Natural Science Foundation of China (62472174, 62372179). Bin Yang is the corresponding author of the work.

\bibliography{main}
\bibliographystyle{icml2025}

%%%%%%%%%%%%%%%%%%%%%%%%%%%%%%%%%%%%%%%%%%%%%%%%%%%%%%%%%%%%%%%%%%%%%%%%%%%%%%%
%%%%%%%%%%%%%%%%%%%%%%%%%%%%%%%%%%%%%%%%%%%%%%%%%%%%%%%%%%%%%%%%%%%%%%%%%%%%%%%
% APPENDIX
%%%%%%%%%%%%%%%%%%%%%%%%%%%%%%%%%%%%%%%%%%%%%%%%%%%%%%%%%%%%%%%%%%%%%%%%%%%%%%%
%%%%%%%%%%%%%%%%%%%%%%%%%%%%%%%%%%%%%%%%%%%%%%%%%%%%%%%%%%%%%%%%%%%%%%%%%%%%%%%

\clearpage
\appendix
\onecolumn
\section{Theoretical Analyses}
\label{app: theory}
\subsection{The Stability of KalmanNet}
Since the proposed KalmanNet works in a data-driven manner, the floating-point operation error may cause the covariance matrix $\mathrm{P}$ losing positive definiteness, which often occurs in the step~(\ref{cov update}). To mitigate this, we utilize a numerically stable form for this step.
\begin{theorem}
    The positive-definiteness of covariance matrix $P_k$ during the update step $\mathrm{P}_k = (I - K_kH_k)\hat{\mathrm{P}}_k$ can be retained through a numerically stable form:
 \begin{gather}
    \mathrm{P}_k = \frac{1}{2}(\mathrm{P}_k + \mathrm{P}_k^T),\label{symmetry}\\
     \mathrm{P}_k^{dual} = (I - K_k H_k)\hat{\mathrm{P}}_k(I -K_k H_k)^T + K_k R_k K_k^T\label{postive-definite}
  \end{gather}

\end{theorem}
\begin{proof} 
The goal is to demonstrate the equivalence of the numerically stable form and original form: $\mathrm{P}_k^{dual}=\mathrm{P}_k$.
   \begin{align} 
        \mathrm{P}_k^{dual}&= (I - K_k H_k)\hat{\mathrm{P}}_k (I -K_k H_k)^T + K_k R_k K_k^T,\notag\\
         &= (I - K_k H_k)\hat{\mathrm{P}}_k - (I - K_k H_k)\hat{\mathrm{P}}_k H_k ^T K_k^T + K_k R_k K_k^T,\notag\\
         &= (I - K_k H_k)\hat{\mathrm{P}}_k - \hat{\mathrm{P}}_k H_k^T K_k^T + K_k H_k \hat{\mathrm{P}}_k H_k^T K_k^T + K_k R_k K_k^T,\notag\\
         &= (I - K_k H_k)\hat{\mathrm{P}}_k - \hat{\mathrm{P}}_k H_k^T K_k^T + K_k(H_k \hat{\mathrm{P}}_k H_k^T + R_k)K_k^T, \notag\\
         K_k&=\hat{\mathrm{P}}_k H_k^T (H_k \hat{\mathrm{P}}_k H_k^T + R)^{-1},\notag\\
         \mathrm{P}_k^{dual} &= (I - K_k H_k)\hat{\mathrm{P}}_k - \hat{\mathrm{P}}_k H_k^T K_k^T + \hat{\mathrm{P}}_k H_k^T K_k^T,\notag\\
         &= (I - K_k H_k)\hat{\mathrm{P}}_k = \mathrm{P}_k,\notag
     \end{align}
 $\mathrm{P}_k^{dual}$ is numerically equivalent to the original $\mathrm{P}_k$.
\end{proof}
where (\ref{symmetry}) ensures the symmetry, (\ref{postive-definite}) stabilizes the positive-definiteness by decomposing the formula into the sum of two positive definite terms, which better ensures positive definiteness during floating operation. 

\subsection{The Convergence of $K^2$VAE}
Since $K^2$VAE models a linear dynamical system in the measurement space where the Koopman Operator serves as the state transition equation, we hope that the convergence state of the KalmanNet does not violate the assumptions of Koopman Theory. In $K^2$VAE, we meticulously design the KalmanNet by making it gradually converge to the Koopman Operator in the forecasting horizon.
\begin{theorem}
 When $U \to \mathbf{0}$, the state transition equation of the KalmanNet gradually converges to the Koopman Operator.
\end{theorem}
\begin{proof}
Under the assumptions of Koopman Theory, $u_k\to \mathbf{0}$ means the linear system constructed by Koopman Operator has little bias in the current measurement space, which leads to high performance in prediction. Meanwhile, the Predict and Update Steps of $z_t$ are converted to:
\begin{align}
    \textit{Predict:} \ \ \ \hat{z}_k &= Az_{k-1}\\
    \textit{Update:} \ \ \ K_k &= \hat{\mathrm{P}}_k H^T(H\hat{\mathrm{P}}_k H^T + R)^{-1},\\
    z_k &=  \hat{z}_{k} + K_k(\hat{x}_k^H - H \hat{z}_{k})
 \end{align}
In this basic case, the state transition equation obeys the basic assumptions of Koopman Theory and $A$ can be treated as a ``fine-tuned'' Koopman Operator $\mathcal{K}$ which is enhanced by the Kalman gain and has stronger generalization ability.

We then consider the special case that KalmanNet fully relys on the observation $\hat{x}_k^H$ from the linear system constructed by Koopman Operator $\mathcal{K}$, thus $H \to I, A \to \mathbf{0}, R \to \mathbf{0}$, the Predict and Update Steps are converted to:
\begin{align}
     \textit{Predict:} \ \ \ \hat{z}_k &= \mathbf{0}\\
    \textit{Update:} \ \ \ z_k &= \hat{x}_k^H
 \end{align}
The system constructed by KalmanNet can be treated as $z_t = \mathcal{K}z_{t-1}$ equivalent to the original Koopman Operator.
\end{proof}

\section{Related Works}
\subsection{Time Series Forecasting}
Time series forecasting (TSF) predicts future observations based on historical observations. TSF methods are mainly categorized into three distinct approaches: (1) statistical learning-based methods, (2) machine learning-based methods, and (3) deep learning-based methods. Early TSF methods primarily rely on statistical learning approaches such as ARIMA~\citep{box1970distribution}, ETS~\citep{hyndman2008forecasting}, and VAR~\citep{godahewa2021monash}. With advancements in machine learning, methods like XGBoost~\citep{chen2016xgboost}, Random Forests~\citep{breiman2001random}, and LightGBM~\citep{ke2017lightgbm} gain popularity for handling nonlinear patterns. However, these methods still require manual feature engineering and model design. Recently, deep learning has made impressive progress in natural language processing~\citep{chen2024gim,zhang2024can,wang2024enhancing,wu2024prompt,wu2025generative}, computer vision~\citep{zhang2025rethinking,zhang2024distilling,wuimgfu,cui2024real,yu2024scnet,li2025cp2m,li2025ddunet,yu2024ichpro}, multimodal~\citep{zhang2025imdprompter,cui2024correlation,jing2023category,jing2023multimodal}, and other aspects~\citep{wang2025tois,yu2025prnet,cui2025detection, yi2025score,li2023daanet,DBLP:conf/icde/00010GY0HXJ24,DBLP:journals/pacmmod/Wu0ZG0J23,DBLP:journals/pvldb/ZhaoGCHZY23, AimTS}. Studies have shown that learned features may perform better than human-designed features~\citep{qiu2025easytime,qiu2025comprehensive,liu2025calf,yu2024ginar}. Leveraging the representation learning of deep neural networks (DNNs), many deep learning-based methods emerge. TimesNet~\citep{wu2022timesnet} and SegRNN~\citep{lin2023segrnn} model time series as vector sequences, using CNNs or RNNs to capture temporal dependencies. Additionally, Transformer architectures, including Informer~\citep{zhou2021informer}, Dsformer~\citep{yu2023dsformer}, TimeFilter~\citep{hu2025timefilter}, TimeBridge~\citep{liu2025timebridge}, PDF~\citep{dai2024period}, Triformer~\citep{Triformer}, PatchTST~\citep{nie2022time}, ROSE~\citep{wang2025rose}, LightGTS~\citep{wang2025lightgts}, and MagicScaler~\citep{magicscaler} capture complex relationships between time points more accurately, significantly improving forecasting performance. MLP-based methods, including DUET~\citep{qiu2025duet}, AMD~\citep{hu2025adaptive}, SparseTSF~\citep{lin2024sparsetsf}, CycleNet~\citep{lincyclenet},  NLinear~\citep{zeng2023transformers}, and DLinear~\citep{zeng2023transformers}, adopt simpler architectures with fewer parameters but still achieve highly competitive forecasting accuracy. 

\section{Experimental Details}
\subsection{Datasets}
\label{appendix DATASETS}
In order to comprehensively evaluate the performance of $K^2$VAE, we conduct experiments on 8 datasets of short-term forecasting and 9 datasets of long-term forecasting under the framework of ProbTS~\cite{ProbTS}, a comprehensive benchmark used to evaluate probabilistic forecasting tasks. Specifically, we use the datasets ETTh1-S, ETTh2-S, ETTm1-S, ETTm2-S, Electricity-S, Solar-S, Traffic-S, and Exchange-S for short-term forecasting, of which the context length is equivalent to forecasting horizon with $T=L=30$ for Exchange-S and $T=L=24$ for the others. For long-term forecasting, we use the datasets ETTh1-L, ETTh2-L, ETTm1-L, ETTm2-L, Electricity-L, Traffic-L, Exchange-L, Weather-L, and ILI-L with prediction length $L \in \{24, 36, 48, 60\}$ for ILI-L and $L \in \{96, 192, 336, 720\}$ for the others. Note that we fix the context length of all the models with $T=36$ for ILI-L and $T=96$ for the others to ensure a fair comparison. \textit{Please note that although datasets with the same prefix may appear similar, they are not necessarily the same. For example, Electricity-L and Electricity-S are not the same dataset, despite both having the prefix ``Electricity." The datasets we use are all derived from the authoritative probabilistic forecasting benchmark, ProbTS. Furthermore, due to the differences in long-term and short-term tasks, the datasets used for long-term and short-term forecasting in ProbTS and $K^2$VAE are also different.} Table~\ref{Dataset Summary} lists statistics of the multivariate time series datasets.

% we evaluated 25 multivariate datasets provided in TFB~\cite{qiu2024tfb}, which cover 10 domains. The frequencies vary from 5 minutes to 1 month, the range of feature dimensions varies from 5 to 2,000, and the sequence length varies from 728 to 57,600. This substantial diversity of the datasets enables comprehensive studies of forecasting methods. Table~\ref{Multivariate datasets} lists statistics of the 25 multivariate time series datasets.

\begin{table*}[ht]
\centering
\caption{Dataset Summary.}
\resizebox{\textwidth}{!}{
\begin{tabular}{l|c|cccrl}
\toprule
 \textbf{Horizon} & \textbf{Dataset} & \textbf{\#var.} & \textbf{range} & \textbf{freq.} & \textbf{timesteps} & \textbf{Description}  \\
 \midrule
 \multirow{7}{*}{\textbf{Long-term}} &  ETTh1/h2-L  & 7 & $\mathbb{R}^+$ & H & 17,420 & Electricity transformer temperature per hour \\
 &  ETTm1/m2-L    & 7 & $\mathbb{R}^+$ & 15min & 69,680 & Electricity transformer temperature every 15 min  \\
 &  Electricity-L & 321 & $\mathbb{R}^+$ & H & 26,304 & Electricity consumption (Kwh) \\
 &  Traffic-L     & 862 & (0,1) & H & 17,544 & Road occupancy rates \\
 &  Exchange-L    & 8 & $\mathbb{R}^+$ & Busi. Day & 7,588 & Daily exchange rates of 8 countries \\
 &  ILI-L         & 7 & (0,1) & W & 966 & Ratio of patients seen with influenza-like illness \\
 &  Weather-L     & 21 & $\mathbb{R}^+$ & 10min & 52,696 & Local climatological data \\
  % &  CAISO     & 10 & $\mathbb{R}^+$ & H & 74,472 & Electricity load in 10 zones of California \\
  %  &  NordPool     & 18 & $\mathbb{R}^+$ & H & 70,128 & Energy production volume in 18 European countries \\
 \midrule
 \multirow{6}{*}{\textbf{Short-term}}   & ETTh1/h2-S & 7 & $\mathbb{R}^+$ & H & 17,420 & Electricity transformer temperature per hour \\
 &  ETTm1/m2-S    & 7 & $\mathbb{R}^+$ & 15min & 69,680 & Electricity transformer temperature every 15 min  \\
 &  Exchange-S & 8   & $\mathbb{R}^+$ & Busi. Day & 6,071 & Daily exchange rates of 8 countries \\
  &  Solar-S    & 137 & $\mathbb{R}^+$ & H & 7,009 & Solar power production records  \\
  &  Electricity-S    & 370   & $\mathbb{R}^+$ & H & 5,833 & Electricity consumption \\
  &  Traffic-S    & 963   & (0,1) & H & 4,001 & Road occupancy rates  \\

\bottomrule
\end{tabular}}
\label{Dataset Summary}
\end{table*}

\renewcommand{\arraystretch}{1}
\begin{table*}[t!]
\centering
\caption{Code repositories for baselines.}
\label{Code repositories for baselines}
\resizebox{0.75\columnwidth}{!}{
\footnotesize
\begin{tabular}{@{}c|c}
\toprule
  Baselines & Code repositories  \\
\midrule
Koopa & \href{https://github.com/thuml/koopa}{\textcolor{blue}{https://github.com/thuml/koopa}}  \\\addlinespace\cline{1-2}\addlinespace
iTransformer & \href{https://github.com/thuml/iTransformer}{\textcolor{blue}{https://github.com/thuml/iTransformer}}  \\\addlinespace\cline{1-2}\addlinespace
FITS & \href{https://github.com/VEWOXIC/FITS}{\textcolor{blue}{https://github.com/VEWOXIC/FITS}}  \\\addlinespace\cline{1-2}\addlinespace
PatchTST & \href{https://github.com/yuqinie98/PatchTST}{\textcolor{blue}{https://github.com/yuqinie98/PatchTST}}  \\\addlinespace\cline{1-2}\addlinespace 
TSDiff & \href{https://github.com/amazon-science/unconditional-time-series-diffusion}{\textcolor{blue}{https://github.com/amazon-science/unconditional-time-series-diffusion}}  \\\addlinespace\cline{1-2}\addlinespace 
$D^3$VAE & \href{https://github.com/PaddlePaddle/PaddleSpatial/tree/main/research/D3VAE}{\textcolor{blue}{https://github.com/PaddlePaddle/PaddleSpatial/tree/main/research/D3VAE}}  \\\addlinespace\cline{1-2}\addlinespace 
GRU NVP & \href{https://github.com/zalandoresearch/pytorch-ts}{\textcolor{blue}{https://github.com/zalandoresearch/pytorch-ts}}  \\\addlinespace\cline{1-2}\addlinespace
GRU MAF & \href{https://github.com/zalandoresearch/pytorch-ts}{\textcolor{blue}{https://github.com/zalandoresearch/pytorch-ts}}  \\\addlinespace\cline{1-2}\addlinespace 
Trans MAF & \href{https://github.com/zalandoresearch/pytorch-ts}{\textcolor{blue}{https://github.com/zalandoresearch/pytorch-ts}}  \\\addlinespace\cline{1-2}\addlinespace 
TimeGrad & \href{https://github.com/yuqinie98/PatchTST}{\textcolor{blue}{https://github.com/yuqinie98/PatchTST}}  \\\addlinespace\cline{1-2}\addlinespace 
CSDI & \href{https://github.com/ermongroup/CSDI}{\textcolor{blue}{https://github.com/ermongroup/CSDI}}  \\\addlinespace\cline{1-2}\addlinespace 
$K^2$VAE (ours) & \href{https://github.com/decisionintelligence/K2VAE}{\textcolor{blue}{https://github.com/decisionintelligence/K2VAE}}  \\
\bottomrule
\end{tabular}
}
\end{table*}

\subsection{Baselines}
\label{appendix BASELINES}
In the realm of probabilistic time series forecasting, numerous models have surfaced in recent years. Following the experimental setting in ProbTS, we compare $K^2$VAE with 11 strong baselines including 4 point forecasting models: FITS, PatchTST, iTransformer, Koopa, and 7 generative models: TSDiff, $D^3$VAE, GRU NVP, GRU MAF, Trans MAF, TimeGrad, CSDI on both short-term and long-term probabilistic forecasting scenarios. The specific code repositories for each of these models--see Table~\ref{Code repositories for baselines}.

\subsection{Evaluation Metrics}
\label{appendix Metrics}
We use two commonly-used metrics NMAE (Normalized Mean Absolute Error) and CPRS (Continuous Ranked Probability Score) in ProbTS~\cite{ProbTS} to evaluate the probabilistic forecasts.

\paragraph{Normalized Mean Absolute Error (NMAE)}
The Normalized Mean Absolute Error (NMAE) is a normalized version of the MAE, which is dimensionless and facilitates the comparability of the error magnitude across different datasets or scales. The mathematical representation of NMAE is given by:
\begin{align}
\textrm{NMAE} = \frac{\sum^K_{k=1} \sum^T_{t=1} |x^k_{t} - \hat{x}^k_{t}|}{\sum^K_{k=1} \sum^T_{t=1} |x^k_{t}|}. 
 \end{align}

\paragraph{Continuous Ranked Probability Score (CRPS)}
The Continuous Ranked Probability Score (CRPS)~\citep{CRPS} quantifies the agreement between a cumulative distribution function (CDF) $F$ and an observation $x$, represented as:
\begin{align}
\textrm{CRPS} = \int_{\mathds{R}} (F(z) - \mathds{I} \{ x \leq z \})^2 dz, 
\end{align}
where $\mathds{I} \{ x \leq z \}$ denotes the indicator function, equating to one if $x \leq z$ and zero otherwise.

Being a proper scoring function, CRPS reaches its minimum when the predictive distribution $F$ coincides with the data distribution. When using the empirical CDF of $F$, denoted as $\hat{F} (z) = \frac{1}{n} \sum^n_{i=1} \mathds{I} \{ X_i \leq z \} $, where $n$ represents the number of samples $X_i \sim F$, CRPS can be precisely calculated from the simulated samples of the conditional distribution $p_\theta (\mathbf{x}_t | \mathbf{h}_t)$. In our practice, 100 samples are employed to estimate the empirical CDF.

\subsection{Implementation Details}
\label{appendix Implementation Details}
For each method, we adhere to the hyper-parameter as specified in their original papers. Additionally, we perform hyper-parameter searches across multiple sets, with a limit of 8 sets. The optimal result is then selected from these evaluations, contributing to a comprehensive and unbiased assessment of each method's performance. 

The \textit {``Drop Last"} issue is reported by several researchers~\cite{qiu2024tfb, qiu2025tab, li2025TSMF-Bench}. That is, in some previous works evaluating the model on test set with drop-last=True setting may cause additional errors related to test batch size. In our experiment, to ensure fair comparison in the future, we set the drop last to False for all baselines to avoid this issue.

All experiments are conducted using PyTorch~\cite{paszke2019pytorch} in Python 3.10 and execute on an NVIDIA Tesla-A800 GPU. The training process is guided by the $\mathcal{L}_{ELBO}$ and $\mathcal{L}_{Rec}$, employing the ADAM optimizer. Initially, the batch size is set to 32, with the option to reduce it by half (to a minimum of 8) in case of an Out-Of-Memory (OOM) situation. To ensure reproducibility and facilitate experimentation, datasets and code are available at: \href{https://github.com/decisionintelligence/K2VAE}{https://github.com/decisionintelligence/K2VAE}.

\clearpage
\subsection{Full Results}
We provide all the main results of LPTSF in Table~\ref{tab:long_term_fore_CRPS} and Table~\ref{tab:long_term_fore_NMAE}, covering all four horizons ($L \in \{96, 192, 336, 720\}$) on 9 real world datasets. The results show that $K^2$VAE achieves a comprehensive lead in long-term prediction tasks, not only outperforming generative models specialized for probabilistic prediction but also demonstrating significant advantages compared to long-term point-based prediction models.

We provide the complete results of ablation studies in Table~\ref{tab: complete koopman operator}--\ref{tab: complete kalman filter}. We compare the different variants under various tasks across different horizons, empirical results demonstrate that $K^2$VAE adopts the most appropriate design.

We also provide the complete efficiency analyses under different forecasting scenarios, which demonstrates that our proposed $K^2$VAE exhibits low memory overhead, fast inference speed, and high accuracy across various tasks. Compared to generative models such as those diffusion-based or flow-based models, $K^2$VAE is both more precise and lightweight. 
\label{app: full results}
 \begin{table}[!htbp]
    \centering
    \setlength\tabcolsep{2pt}
    \scriptsize
    \caption{Results of CRPS ($\textrm{mean}_{\textrm{std}}$) on long-term forecasting scenarios, each containing five independent runs with different seeds. The context length is set to 36 for the ILI-L dataset and 96 for the others. Lower CRPS values indicate better predictions. The means and standard errors are based on 5 independent runs of retraining and evaluation. \textcolor{red}{\textbf{Red}}: the best, \textcolor{blue}{\underline{Blue}}: the 2nd best.}
    \resizebox{\textwidth}{!}{
  \begin{tabular}{c|c|c|c|c|c|c|c|c|c|c|c|c}
    \toprule
        Dataset & Horizon & Koopa & iTransformer & FITS & PatchTST & GRU MAF & Trans MAF & TSDiff & CSDI & TimeGrad & GRU NVP & $K^2$VAE  \\ \midrule
        \multirow{4}{*}{ETTm1-L} & 96 &$0.285\scriptstyle\pm0.018$ & $0.301\scriptstyle\pm0.033$ & $0.267\scriptstyle\pm0.023$ & $0.261\scriptstyle\pm0.051$ & $0.295\scriptstyle\pm0.055$ & $0.313\scriptstyle\pm0.045$ & $0.344\scriptstyle\pm0.050$ & $\textcolor{blue}{\underline{0.236}\scriptstyle\pm0.006}$ & $0.522\scriptstyle\pm0.105$ & $0.383\scriptstyle\pm0.053$ & $\textcolor{red}{\textbf{0.232}\scriptstyle\pm0.010}$  \\ 
        ~ & 192 &$0.289\scriptstyle\pm0.024$& $0.314\scriptstyle\pm0.023$ & $\textcolor{blue}{\underline{0.261}\scriptstyle\pm0.022}$ & $0.275\scriptstyle\pm0.030$ & $0.389\scriptstyle\pm0.033$ & $0.424\scriptstyle\pm0.029$ & $0.345\scriptstyle\pm0.035$ & $0.291\scriptstyle\pm0.025$ & $0.603\scriptstyle\pm0.092$ & $0.396\scriptstyle\pm0.030$ & $\textcolor{red}{\textbf{0.259}\scriptstyle\pm0.013}$  \\ 
        ~ & 336 &$0.286\scriptstyle\pm0.035$& $0.311\scriptstyle\pm0.029$ & $\textcolor{blue}{\underline{0.275}\scriptstyle\pm0.030}$ & $0.285\scriptstyle\pm0.028$ & $0.429\scriptstyle\pm0.021$ & $0.481\scriptstyle\pm0.019$ & $0.462\scriptstyle\pm0.043$ & $0.322\scriptstyle\pm0.033$ & $0.601\scriptstyle\pm0.028$ & $0.486\scriptstyle\pm0.032$ & $\textcolor{red}{\textbf{0.262}\scriptstyle\pm0.030}$  \\ 
        ~ & 720 &$\textcolor{blue}{\underline{0.295}\scriptstyle\pm0.027}$ & $0.455\scriptstyle\pm0.021$ & $0.305\scriptstyle\pm0.024$ & $0.304\scriptstyle\pm0.029$ & $0.536\scriptstyle\pm0.033$ & $0.688\scriptstyle\pm0.043$ & $0.478\scriptstyle\pm0.027$ & $0.448\scriptstyle\pm0.038$ & $0.621\scriptstyle\pm0.037$ & $0.546\scriptstyle\pm0.036$ & $\textcolor{red}{\textbf{0.294}\scriptstyle\pm0.026}$  \\ \midrule
        \multirow{4}{*}{ETTm2-L} & 96 &$0.178\scriptstyle\pm0.023$ & $0.181\scriptstyle\pm0.031$ & $0.162\scriptstyle\pm0.053$ & $0.142\scriptstyle\pm0.034$ & $0.177\scriptstyle\pm0.024$ & $0.227\scriptstyle\pm0.013$ & $0.175\scriptstyle\pm0.019$ & $\textcolor{red}{\textbf{0.115}\scriptstyle\pm0.009}$ & $0.427\scriptstyle\pm0.042$ & $0.319\scriptstyle\pm0.044$ & $\textcolor{blue}{\underline{0.126}\scriptstyle\pm0.007}$  \\
        ~ & 192 &$0.185\scriptstyle\pm0.014$& $0.190\scriptstyle\pm0.010$ &$0.185\scriptstyle\pm0.053$ & $0.172\scriptstyle\pm0.023$ & $0.411\scriptstyle\pm0.026$ & $0.253\scriptstyle\pm0.037$ & $0.255\scriptstyle\pm0.029$ & $\textcolor{red}{\textbf{0.147}\scriptstyle\pm0.008}$ & $0.424\scriptstyle\pm0.061$ & $0.326\scriptstyle\pm0.025$ & $\textcolor{blue}{\underline{0.148}\scriptstyle\pm0.009}$  \\ 
        ~ & 336 &$0.198\scriptstyle\pm0.015$& $0.206\scriptstyle\pm0.055$ & $0.218\scriptstyle\pm0.053$ & $0.195\scriptstyle\pm0.042$ & $0.377\scriptstyle\pm0.023$ & $0.253\scriptstyle\pm0.013$ & $0.328\scriptstyle\pm0.047$ & $\textcolor{blue}{\underline{0.190}\scriptstyle\pm0.018} $& $0.469\scriptstyle\pm0.049$ & $0.449\scriptstyle\pm0.145$ & $\textcolor{red}{\textbf{0.164}\scriptstyle\pm0.010}$  \\ 
        ~ & 720 &$0.233\scriptstyle\pm0.025$& $0.311\scriptstyle\pm0.024$ & $0.449\scriptstyle\pm0.034$ & $\textcolor{blue}{\underline{0.229}\scriptstyle\pm0.036}$ & $0.272\scriptstyle\pm0.029$ & $0.355\scriptstyle\pm0.043$ & $0.344\scriptstyle\pm0.046$ & $0.239\scriptstyle\pm0.035$ & $0.470\scriptstyle\pm0.054$ & $0.561\scriptstyle\pm0.273$ & $\textcolor{red}{\textbf{0.221}\scriptstyle\pm0.023}$  \\ \midrule
        \multirow{4}{*}{ETTh1-L} & 96 &$0.307\scriptstyle\pm0.033$ & $\textcolor{blue}{\underline{0.292}\scriptstyle\pm0.032}$ & $0.294\scriptstyle\pm0.023$ & $0.312\scriptstyle\pm0.036$ & $0.293\scriptstyle\pm0.037$ & $0.333\scriptstyle\pm0.045$ & $0.395\scriptstyle\pm0.052$ & $0.437\scriptstyle\pm0.018$ & $0.455\scriptstyle\pm0.046$ & $0.379\scriptstyle\pm0.030$ & $\textcolor{red}{\textbf{0.264}\scriptstyle\pm0.020}$  \\ 
        ~ & 192 &$0.301\scriptstyle\pm0.014$ & $\textcolor{blue}{\underline{0.298}\scriptstyle\pm0.020}$ & $0.304\scriptstyle\pm0.028$ & $0.313\scriptstyle\pm0.034$ & $0.348\scriptstyle\pm0.075$ & $0.351\scriptstyle\pm0.063$ & $0.467\scriptstyle\pm0.044$ & $0.496\scriptstyle\pm0.051$ & $0.516\scriptstyle\pm0.038$ & $0.425\scriptstyle\pm0.019$ & $\textcolor{red}{\textbf{0.290}\scriptstyle\pm0.016}$  \\
        ~ & 336 &$\textcolor{blue}{\underline{0.312}\scriptstyle\pm0.019}$& $0.327\scriptstyle\pm0.043$ & $0.318\scriptstyle\pm0.023$ & $0.319\scriptstyle\pm0.035$ & $0.377\scriptstyle\pm0.026$ & $0.371\scriptstyle\pm0.031$ & $0.450\scriptstyle\pm0.027$ & $0.454\scriptstyle\pm0.025$ & $0.512\scriptstyle\pm0.026$ & $0.458\scriptstyle\pm0.054$ & $\textcolor{red}{\textbf{0.308}\scriptstyle\pm0.021}$ \\ 
        ~ & 720 &$\textcolor{blue}{\underline{0.318}\scriptstyle\pm0.009}$& $0.350\scriptstyle\pm0.019$ & $0.348\scriptstyle\pm0.025$ & $0.323\scriptstyle\pm0.020$& $0.393\scriptstyle\pm0.043$ & $0.363\scriptstyle\pm0.053$ & $0.516\scriptstyle\pm0.027$ & $0.528\scriptstyle\pm0.012$ & $0.523\scriptstyle\pm0.027$ & $0.502\scriptstyle\pm0.039$ & $\textcolor{red}{\textbf{0.314}\scriptstyle\pm0.011}$  \\ \midrule
        \multirow{4}{*}{ETTh2-L} & 96 &$0.199\scriptstyle\pm0.012$ & $0.185\scriptstyle\pm0.013$ & $0.187\scriptstyle\pm0.011$ & $0.197\scriptstyle\pm0.021$ & $0.239\scriptstyle\pm0.019$ & $0.263\scriptstyle\pm0.020$ & $0.336\scriptstyle\pm0.021$ & $\textcolor{blue}{\underline{0.164}\scriptstyle\pm0.013}$ & $0.358\scriptstyle\pm0.026$ & $0.432\scriptstyle\pm0.141$ & $\textcolor{red}{\textbf{0.162}\scriptstyle\pm0.009}$  \\ 
        ~ & 192 &$0.198\scriptstyle\pm0.022$& $0.199\scriptstyle\pm0.019$ & $\textcolor{blue}{\underline{0.195}\scriptstyle\pm0.022}$ & $0.204\scriptstyle\pm0.055$ & $0.313\scriptstyle\pm0.034$ & $0.273\scriptstyle\pm0.024$ & $0.265\scriptstyle\pm0.043$ & $0.226\scriptstyle\pm0.018$ & $0.457\scriptstyle\pm0.081$ & $0.625\scriptstyle\pm0.170$ &$\textcolor{red}{\textbf{0.186}\scriptstyle\pm0.018}$ \\ 
        ~ & 336 &$0.262\scriptstyle\pm0.019$& $0.271\scriptstyle\pm0.033$ & $\textcolor{red}{\textbf{0.246}\scriptstyle\pm0.044}$ & $0.277\scriptstyle\pm0.054$ & $0.376\scriptstyle\pm0.034$ & $0.265\scriptstyle\pm0.042$ & $0.350\scriptstyle\pm0.031$ & $0.274\scriptstyle\pm0.022$ & $0.481\scriptstyle\pm0.078$ & $0.793\scriptstyle\pm0.319$ & $\textcolor{blue}{\underline{0.257}\scriptstyle\pm0.023}$  \\ 
        ~ & 720 &$\textcolor{blue}{\underline{0.293}\scriptstyle\pm0.026}$& $0.542\scriptstyle\pm0.015$ & $0.314\scriptstyle\pm0.022$ & $0.304\scriptstyle\pm0.018$ & $0.990\scriptstyle\pm0.023$ & $0.327\scriptstyle\pm0.033$ & $0.406\scriptstyle\pm0.056$ & $0.302\scriptstyle\pm0.040$ & $0.445\scriptstyle\pm0.016$ & $0.539\scriptstyle\pm0.090$ & $\textcolor{red}{\textbf{0.280}\scriptstyle\pm0.014}$  \\ \midrule
        \multirow{4}{*}{Electricity-L} & 96 & $0.110\scriptstyle\pm0.004$& $0.102\scriptstyle\pm0.004$ & $0.105\scriptstyle\pm0.006$ & $0.126\scriptstyle\pm0.005$ & $\textcolor{blue}{\underline{0.083}\scriptstyle\pm0.009}$ & $0.088\scriptstyle\pm0.014$ & $0.344\scriptstyle\pm0.006$ & $0.153\scriptstyle\pm0.137$ & $0.096\scriptstyle\pm0.002$ & $0.094\scriptstyle\pm0.003$ & $\textcolor{red}{\textbf{0.073}\scriptstyle\pm0.002}$  \\ 
        ~ & 192 &$0.109\scriptstyle\pm0.011$ & $0.104\scriptstyle\pm0.014$ & $0.112\scriptstyle\pm0.104$ & $0.123\scriptstyle\pm0.032$ & $\textcolor{blue}{\underline{0.093}\scriptstyle\pm0.024}$ & $0.097\scriptstyle\pm0.009$ & $0.345\scriptstyle\pm0.006$ & $0.200\scriptstyle\pm0.094$ & $0.100\scriptstyle\pm0.004$ & $0.097\scriptstyle\pm0.002$ & $\textcolor{red}{\textbf{0.080}\scriptstyle\pm0.004}$  \\
        ~ & 336 &$0.121\scriptstyle\pm0.011$& $0.104\scriptstyle\pm0.010$ & $0.111\scriptstyle\pm0.014$ & $0.131\scriptstyle\pm0.024$ & $\textcolor{blue}{\underline{0.095}\scriptstyle\pm0.001}$ & - & $0.462\scriptstyle\pm0.054$ & - & $0.102\scriptstyle\pm0.007$ & $0.099\scriptstyle\pm0.001$ & $\textcolor{red}{\textbf{0.054}\scriptstyle\pm0.001}$  \\ 
        ~ & 720 &$0.113\scriptstyle\pm0.018$ & $0.109\scriptstyle\pm0.044$ &  $0.115\scriptstyle\pm0.024$ & $0.127\scriptstyle\pm0.015$ & $\textcolor{blue}{\underline{0.106}\scriptstyle\pm0.007}$ & - & $0.478\scriptstyle\pm0.005$ & - & $0.108\scriptstyle\pm0.003$ & $0.114\scriptstyle\pm0.013$ & $\textcolor{red}{\textbf{0.057}\scriptstyle\pm0.005}$  \\ \midrule
        \multirow{4}{*}{Traffic-L} & 96 &$0.297\scriptstyle\pm0.019$& $0.256\scriptstyle\pm0.004$ & $0.258\scriptstyle\pm0.004$ & $0.194\scriptstyle\pm0.002$ & $0.215\scriptstyle\pm0.003$ & $0.208\scriptstyle\pm0.004$ & $0.294\scriptstyle\pm0.003$ & - & $0.202\scriptstyle\pm0.004$ & $\textcolor{blue}{\underline{0.187}\scriptstyle\pm0.002}$ & $\textcolor{red}{\textbf{0.086}\scriptstyle\pm0.001}$  \\ 
        ~ & 192 &$0.308\scriptstyle\pm0.009$& $0.250\scriptstyle\pm0.002$ & $0.275\scriptstyle\pm0.003$ & $0.198\scriptstyle\pm0.004$ & - & - & $0.306\scriptstyle\pm0.004$ & - & $0.208\scriptstyle\pm0.003 $& $\textcolor{blue}{\underline{0.192}\scriptstyle\pm0.001}$ & $\textcolor{red}{\textbf{0.088}\scriptstyle\pm0.002}$  \\ 
        ~ & 336 &$0.334\scriptstyle\pm0.017$ & $0.261\scriptstyle\pm0.001$ & $0.327\scriptstyle\pm0.001$ & $0.204\scriptstyle\pm0.002$ & - & - & $0.317\scriptstyle\pm0.006$ & - & $0.213\scriptstyle\pm0.003$ & $\textcolor{blue}{\underline{0.201}\scriptstyle\pm0.004}$ & $\textcolor{red}{\textbf{0.195}\scriptstyle\pm0.003}$  \\ 
        ~ & 720 &$0.358\scriptstyle\pm0.022$& $0.284\scriptstyle\pm0.004$ & $0.374\scriptstyle\pm0.004$ & $0.214\scriptstyle\pm0.001$ & - & - & $0.391\scriptstyle\pm0.002$ & - & $0.220\scriptstyle\pm0.002$ & $\textcolor{blue}{\underline{0.211}\scriptstyle\pm0.004}$ & $\textcolor{red}{\textbf{0.200}\scriptstyle\pm0.001}$  \\ \midrule
        \multirow{4}{*}{Weather-L} & 96 &$0.132\scriptstyle\pm0.008$& $0.131\scriptstyle\pm0.011$ & $0.210\scriptstyle\pm0.013$ & $0.131\scriptstyle\pm0.007$ & $0.139\scriptstyle\pm0.008$ & $0.105\scriptstyle\pm0.011$ & $0.104\scriptstyle\pm0.020$ & $\textcolor{red}{\textbf{0.068}\scriptstyle\pm0.008}$ & $0.130\scriptstyle\pm0.017$ & $0.116\scriptstyle\pm0.013 $& $\textcolor{blue}{\underline{0.080}\scriptstyle\pm0.007}$  \\ 
        ~ & 192 &$0.133\scriptstyle\pm0.017$& $0.132\scriptstyle\pm0.018$ & $0.205\scriptstyle\pm0.019$ & $0.131\scriptstyle\pm0.014$ & $0.143\scriptstyle\pm0.020$ & $0.142\scriptstyle\pm0.022$ & $0.134\scriptstyle\pm0.012$ & $\textcolor{red}{\textbf{0.068}\scriptstyle\pm0.006}$ & $0.127\scriptstyle\pm0.019$ & $0.122\scriptstyle\pm0.021$ & $\textcolor{blue}{\underline{0.079}\scriptstyle\pm0.009}$  \\ 
        ~ & 336 &$0.136\scriptstyle\pm0.021$& $0.132\scriptstyle\pm0.010$ & $0.221\scriptstyle\pm0.005$ & $0.137\scriptstyle\pm0.008$ & $0.129\scriptstyle\pm0.012$ & $0.133\scriptstyle\pm0.014$ & $0.137\scriptstyle\pm0.010$ & $\textcolor{blue}{\underline{0.083}\scriptstyle\pm0.002}$ & $0.130\scriptstyle\pm0.006$ & $0.128\scriptstyle\pm0.011$ &$\textcolor{red}{\textbf{0.082}\scriptstyle\pm0.010}$  \\ 
        ~ & 720 &$0.140\scriptstyle\pm0.007$ & $0.133\scriptstyle\pm0.004$ & $0.267\scriptstyle\pm0.003$ & $0.142\scriptstyle\pm0.005$ & $0.122\scriptstyle\pm0.006$ & $0.113\scriptstyle\pm0.004$ & $0.152\scriptstyle\pm0.003$ & $\textcolor{blue}{\underline{0.087}\scriptstyle\pm0.003}$ & $0.113\scriptstyle\pm0.011$ & $0.110\scriptstyle\pm0.004$ & $\textcolor{red}{\textbf{0.084}\scriptstyle\pm0.003}$  \\ \midrule
        \multirow{4}{*}{Exchange-L} & 96 &$0.063\scriptstyle\pm0.006$& $0.061\scriptstyle\pm0.003$ & $0.048\scriptstyle\pm0.004$ & $0.063\scriptstyle\pm0.006$ & $0.026\scriptstyle\pm0.010$ & $0.028\scriptstyle\pm0.002$ &$0.079\scriptstyle\pm0.007$ & $\textcolor{red}{\textbf{0.028}\scriptstyle\pm0.003}$ & $0.068\scriptstyle\pm0.003$ & $0.071\scriptstyle\pm0.006$ & $\textcolor{blue}{\underline{0.031}\scriptstyle\pm0.002}$  \\ 
        ~ & 192 &$0.065\scriptstyle\pm0.020$ & $0.062\scriptstyle\pm0.010$ & $0.049\scriptstyle\pm0.011$ & $0.067\scriptstyle\pm0.008$ & $\textcolor{blue}{\underline{0.034}\scriptstyle\pm0.009}$ & $0.046\scriptstyle\pm0.017$ & $0.093\scriptstyle\pm0.011$ & $0.045\scriptstyle\pm0.003$ & $0.087\scriptstyle\pm0.013$ & $0.068\scriptstyle\pm0.004$ & $\textcolor{red}{\textbf{0.032}\scriptstyle\pm0.010}$  \\ 
        ~ & 336 &$0.072\scriptstyle\pm0.008$& $0.067\scriptstyle\pm0.008$ & $0.052\scriptstyle\pm0.013$ & $0.071\scriptstyle\pm0.017$ & $0.058\scriptstyle\pm0.023$ & $\textcolor{red}{\textbf{0.045}\scriptstyle\pm0.010}$ & $0.081\scriptstyle\pm0.007$ & $0.060\scriptstyle\pm0.004$ & $0.074\scriptstyle\pm0.009$ & $0.072\scriptstyle\pm0.002$ & $\textcolor{blue}{\underline{0.048}\scriptstyle\pm0.004}$  \\ 
        ~ & 720 &$0.091\scriptstyle\pm0.012$& $0.087\scriptstyle\pm0.023$ & $\textcolor{blue}{\underline{0.074}\scriptstyle\pm0.011}$ & $0.097\scriptstyle\pm0.007$ & $0.160\scriptstyle\pm0.019$ & $0.148\scriptstyle\pm0.017$ & $0.082\scriptstyle\pm0.010$ & $0.143\scriptstyle\pm0.020$ & $0.099\scriptstyle\pm0.015$ & $0.079\scriptstyle\pm0.009$ & $\textcolor{red}{\textbf{0.069}\scriptstyle\pm0.005}$  \\ \midrule
        \multirow{4}{*}{ILI-L} & 24 &$0.245\scriptstyle\pm0.018$& $0.212 \scriptstyle{\pm 0.013}$ & $0.233 \scriptstyle{\pm 0.015}$ & $0.312 \scriptstyle{\pm 0.014}$ & $0.097 \scriptstyle{\pm 0.010}$ & $\textcolor{blue}{\underline{0.092} \scriptstyle{\pm 0.019}}$ & $0.228 \scriptstyle{\pm 0.024}$ & $0.250\scriptstyle\pm0.013$ & $0.275\scriptstyle\pm0.047$ & $0.257\scriptstyle\pm0.003$ & $\textcolor{red}{\textbf{0.087}\scriptstyle\pm0.003}$  \\ 
        ~ & 36 &$0.214\scriptstyle\pm0.008$& $0.182 \scriptstyle{\pm 0.016}$ & $0.217 \scriptstyle{\pm 0.023}$ & $0.241 \scriptstyle{\pm 0.021}$ & $0.117 \scriptstyle{\pm 0.017}$ & $\textcolor{blue}{\underline{0.115} \scriptstyle{\pm 0.011}}$ & $0.235 \scriptstyle{\pm 0.010}$ & $0.285\scriptstyle\pm0.010$ & $0.272\scriptstyle\pm0.057$ & $0.281\scriptstyle\pm0.004$ & $\textcolor{red}{\textbf{0.113} \scriptstyle{\pm 0.005}}$  \\ 
        ~ & 48 &$0.271\scriptstyle\pm0.021$& $0.213 \scriptstyle{\pm 0.012}$ & $0.185 \scriptstyle{\pm 0.026}$ & $0.242 \scriptstyle{\pm 0.018}$ & $0.128 \scriptstyle{\pm 0.019}$ & $\textcolor{blue}{\underline{0.133} \scriptstyle{\pm 0.022}}$ & $0.265 \scriptstyle{\pm 0.039}$ & $0.285\scriptstyle\pm0.036$ & $0.295\scriptstyle\pm0.033$ & $0.288\scriptstyle\pm0.008$ & $\textcolor{red}{\textbf{0.124} \scriptstyle{\pm 0.010}}$  \\ 
        ~ & 60 &$0.228\scriptstyle\pm0.022$& $0.222 \scriptstyle{\pm 0.020}$ &$0.211 \scriptstyle{\pm 0.011}$ & $0.233 \scriptstyle{\pm 0.019}$ & $0.172 \scriptstyle{\pm 0.034}$ & $\textcolor{blue}{\underline{0.155} \scriptstyle{\pm 0.018}}$ & $0.263 \scriptstyle{\pm 0.022}$ & $0.283\scriptstyle\pm0.012$ & $0.295\scriptstyle\pm0.083$ & $0.307\scriptstyle\pm0.005$ & $\textcolor{red}{\textbf{0.142} \scriptstyle{\pm 0.008}}$ \\ \bottomrule
        \multicolumn{13}{l}{Due to the excessive time and memory consumption, some results are unavailable in our implementation  and denoted as -.}
    \end{tabular}}
  \label{tab:long_term_fore_CRPS}
  \end{table}   

\begin{table}  
    \centering
    \setlength\tabcolsep{2pt}
    \scriptsize
    \caption{Results of NMAE ($\textrm{mean}_{\textrm{std}}$) on long-term forecasting scenarios, each containing five independent runs with different seeds. The context length is set to 36 for the ILI-L dataset and 96 for the others. Lower NMAE values indicate better predictions. The means and standard errors are based on 5 independent runs of retraining and evaluation. \textcolor{red}{\textbf{Red}}: the best, \textcolor{blue}{\underline{Blue}}: the 2nd best.}
    \resizebox{\textwidth}{!}{
  \begin{tabular}{c|c|c|c|c|c|c|c|c|c|c|c|c}
    \toprule
        Dataset & Horizon & Koopa & iTransformer & FITS & PatchTST & GRU MAF & Trans MAF & TSDiff & CSDI & TimeGrad & GRU NVP & $K^2$VAE \\  \midrule
        \multirow{4}{*}{ETTm1-L} & 96 &$0.362\scriptstyle\pm0.022$ & $0.369\scriptstyle\pm0.029$ & $0.349\scriptstyle\pm0.032$ & $0.329\scriptstyle\pm0.100$ & $0.402\scriptstyle\pm0.087$ & $0.456\scriptstyle\pm0.042$ & $0.441\scriptstyle\pm0.021$ & $\textcolor{blue}{\underline{0.308}\scriptstyle\pm0.005}$ & $0.645\scriptstyle\pm0.129$ & $0.488\scriptstyle\pm0.058$ & $\textcolor{red}{\textbf{0.284}\scriptstyle\pm0.011}$  \\ 
        ~ & 192 &$0.365\scriptstyle\pm0.032$& $0.384\scriptstyle\pm0.041$ & $0.341\scriptstyle\pm0.032$ & $\textcolor{blue}{\underline{0.338}\scriptstyle\pm0.022}$ & $0.476\scriptstyle\pm0.046$ & $0.553\scriptstyle\pm0.012$ & $0.441\scriptstyle\pm0.019$ & $0.377\scriptstyle\pm0.026$ & $0.748\scriptstyle\pm0.084$ & $0.514\scriptstyle\pm0.042$ & $\textcolor{red}{\textbf{0.323}\scriptstyle\pm0.020}$  \\ 
        ~ & 336 &$0.364\scriptstyle\pm0.026$ & $0.380\scriptstyle\pm0.020$ & $0.356\scriptstyle\pm0.022$ & $\textcolor{blue}{\underline{0.344}\scriptstyle\pm0.013}$ & $0.522\scriptstyle\pm0.019$ & $0.590\scriptstyle\pm0.047$ & $0.571\scriptstyle\pm0.033$ & $0.419\scriptstyle\pm0.042$ & $0.759\scriptstyle\pm0.015$ & $0.630\scriptstyle\pm0.029$ &  $\textcolor{red}{\textbf{0.330}\scriptstyle\pm0.014}$  \\ 
        ~ & 720 &$\textcolor{blue}{\underline{0.377}\scriptstyle\pm0.037}$& $0.490\scriptstyle\pm0.038$ & $0.406\scriptstyle\pm0.072$ & $0.382\scriptstyle\pm0.066$ & $0.711\scriptstyle\pm0.081$ & $0.822\scriptstyle\pm0.034$ & $0.622\scriptstyle\pm0.045$ & $0.578\scriptstyle\pm0.051$ & $0.793\scriptstyle\pm0.034$ & $0.707\scriptstyle\pm0.050$ & $\textcolor{red}{\textbf{0.373}\scriptstyle\pm0.032}$  \\ \midrule
        \multirow{4}{*}{ETTm2-L} & 96 &$0.225\scriptstyle\pm0.039$ & $0.221\scriptstyle\pm0.039$ & $0.210\scriptstyle\pm0.040$ & $0.216\scriptstyle\pm0.035$ & $0.212\scriptstyle\pm0.082$ & $0.279\scriptstyle\pm0.031$ & $0.224\scriptstyle\pm0.033$ & $\textcolor{blue}{\underline{0.146}\scriptstyle\pm0.012}$ & $0.525\scriptstyle\pm0.047$ & $0.413\scriptstyle\pm0.059$ & $\textcolor{red}{\textbf{0.144}\scriptstyle\pm0.011}$  \\ 
        ~ & 192 &$0.233\scriptstyle\pm0.026$& $0.229\scriptstyle\pm0.031$ & $0.234\scriptstyle\pm0.038$ & $0.215\scriptstyle\pm0.022$ & $0.535\scriptstyle\pm0.029$ & $0.292\scriptstyle\pm0.041$ & $0.316\scriptstyle\pm0.040$ & $\textcolor{blue}{\underline{0.189}\scriptstyle\pm0.012}$ & $0.530\scriptstyle\pm0.060$ & $0.427\scriptstyle\pm0.033$ & $\textcolor{red}{\textbf{0.170}\scriptstyle\pm0.009}$  \\ 
        ~ & 336 &$0.267\scriptstyle\pm0.023$& $0.245\scriptstyle\pm0.049$ & $0.276\scriptstyle\pm0.019$ & $\textcolor{blue}{\underline{0.234}\scriptstyle\pm0.024}$ & $0.407\scriptstyle\pm0.043$ & $0.309\scriptstyle\pm0.032$ & $0.397\scriptstyle\pm0.051$ & $0.248\scriptstyle\pm0.024$ & $0.566\scriptstyle\pm0.047$ & $0.580\scriptstyle\pm0.169$ & $\textcolor{red}{\textbf{0.187}\scriptstyle\pm0.021}$  \\ 
        ~ & 720 &$0.290\scriptstyle\pm0.033$ & $0.385\scriptstyle\pm0.042$ & $0.540\scriptstyle\pm0.052$ & $\textcolor{blue}{\underline{0.288}\scriptstyle\pm0.034}$ & $0.355\scriptstyle\pm0.048$ & $0.475\scriptstyle\pm0.029$ & $0.416\scriptstyle\pm0.065$ & $0.306\scriptstyle\pm0.040$ & $0.561\scriptstyle\pm0.044$ & $0.749\scriptstyle\pm0.385$ & $\textcolor{red}{\textbf{0.275}\scriptstyle\pm0.035}$  \\ \midrule
        \multirow{4}{*}{ETTh1-L} & 96 &$0.407\scriptstyle\pm0.052$ & $0.386\scriptstyle\pm0.092$ & $0.393\scriptstyle\pm0.142$ & $0.407\scriptstyle\pm0.022$ & $\textcolor{blue}{\underline{0.371}\scriptstyle\pm0.034}$ & $0.423\scriptstyle\pm0.047$ & $0.510\scriptstyle\pm0.029$ & $0.557\scriptstyle\pm0.022$ & $0.585\scriptstyle\pm0.058$ & $0.481\scriptstyle\pm0.037$ & $\textcolor{red}{\textbf{0.336}\scriptstyle\pm0.041}$  \\ 
        ~ & 192 &$0.396\scriptstyle\pm0.022$ & $\textcolor{blue}{\underline{0.388}\scriptstyle\pm0.041}$ & $0.406\scriptstyle\pm0.079$ & $0.405\scriptstyle\pm0.088$ & $0.430\scriptstyle\pm0.022$ & $0.451\scriptstyle\pm0.012$ & $0.596\scriptstyle\pm0.056$ & $0.625\scriptstyle\pm0.065$ & $0.680\scriptstyle\pm0.058$ & $0.531\scriptstyle\pm0.018$ & $\textcolor{red}{\textbf{0.372}\scriptstyle\pm0.023}$  \\ 
        ~ & 336 &$\textcolor{blue}{\underline{0.406}\scriptstyle\pm0.028}$& $0.415\scriptstyle\pm0.022$ & $0.410\scriptstyle\pm0.063$ & $0.412\scriptstyle\pm0.024$ & $0.462\scriptstyle\pm0.049$ & $0.481\scriptstyle\pm0.041$ & $0.581\scriptstyle\pm0.035$ & $0.574\scriptstyle\pm0.026$ & $0.666\scriptstyle\pm0.047$ & $0.580\scriptstyle\pm0.064$ & $\textcolor{red}{\textbf{0.394}\scriptstyle\pm0.022}$  \\ 
        ~ & 720 &$\textcolor{blue}{\underline{0.412}\scriptstyle\pm0.008}$ & $0.449\scriptstyle\pm0.022$ & $0.468\scriptstyle\pm0.012$ & $0.428\scriptstyle\pm0.024$ & $0.496\scriptstyle\pm0.019$ & $0.455\scriptstyle\pm0.025$ & $0.657\scriptstyle\pm0.017$ & $0.657\scriptstyle\pm0.014$ & $0.672\scriptstyle\pm0.015$ & $0.643\scriptstyle\pm0.046$ & $\textcolor{red}{\textbf{0.396}\scriptstyle\pm0.012}$  \\ \midrule
        \multirow{4}{*}{ETTh2-L} & 96 & $0.249\scriptstyle\pm0.015$ & $0.234\scriptstyle\pm0.011$ & $0.243\scriptstyle\pm0.009$ & $0.247\scriptstyle\pm0.028$ & $0.292\scriptstyle\pm0.012$ & $0.345\scriptstyle\pm0.042$ & $0.421\scriptstyle\pm0.033$ & $\textcolor{blue}{\underline{0.214}\scriptstyle\pm0.018}$ & $0.448\scriptstyle\pm0.031$ & $0.548\scriptstyle\pm0.158$ & $\textcolor{red}{\textbf{0.189}\scriptstyle\pm0.010}$  \\ 
        ~ & 192 &$0.249\scriptstyle\pm0.032$ & $\textcolor{blue}{\underline{0.247}\scriptstyle\pm0.040}$ & $0.252\scriptstyle\pm0.022$ & $0.265\scriptstyle\pm0.091$& $0.376\scriptstyle\pm0.112$ & $0.343\scriptstyle\pm0.044$ & $0.339\scriptstyle\pm0.033$ & $0.294\scriptstyle\pm0.027$ & $0.575\scriptstyle\pm0.089$ & $0.766\scriptstyle\pm0.223$ & $\textcolor{red}{\textbf{0.213}\scriptstyle\pm0.021}$  \\ 
        ~ & 336 &$\textcolor{blue}{\underline{0.274}\scriptstyle\pm0.027}$ & $0.297\scriptstyle\pm0.029$& $0.291\scriptstyle\pm0.032$ & $0.314\scriptstyle\pm0.045$ & $0.454\scriptstyle\pm0.057$ & $0.333\scriptstyle\pm0.078$ & $0.427\scriptstyle\pm0.041$ & $0.353\scriptstyle\pm0.028$ & $0.606\scriptstyle\pm0.095$ & $0.942\scriptstyle\pm0.408$ & $\textcolor{red}{\textbf{0.263}\scriptstyle\pm0.039}$ \\ 
        ~ & 720 & $\textcolor{blue}{\underline{0.286}\scriptstyle\pm0.042}$ & $0.667\scriptstyle\pm0.012$ & $0.401\scriptstyle\pm0.022$ & $0.371\scriptstyle\pm0.021$ & $1.092\scriptstyle\pm0.019$ &$0.412\scriptstyle\pm0.020$ & $0.482\scriptstyle\pm0.022$ & $0.382\scriptstyle\pm0.030$ & $0.550\scriptstyle\pm0.018$ & $0.688\scriptstyle\pm0.161$ & $\textcolor{red}{\textbf{0.278}\scriptstyle\pm0.020}$ \\ \midrule
        \multirow{4}{*}{Electricity-L} & 96 &$0.146\scriptstyle\pm0.015$ & $0.134\scriptstyle\pm0.002$ & $0.137\scriptstyle\pm0.002$ & $0.168\scriptstyle\pm0.012$& $\textcolor{blue}{\underline{0.108}\scriptstyle\pm0.009}$ & $0.114\scriptstyle\pm0.010$ &$0.441\scriptstyle\pm0.013$ & $0.203\scriptstyle\pm0.189$ & $0.119\scriptstyle\pm0.003$ & $0.118\scriptstyle\pm0.003$ & $\textcolor{red}{\textbf{0.093}\scriptstyle\pm0.002}$  \\
        ~ & 192 &$0.143\scriptstyle\pm0.023$& $0.137\scriptstyle\pm0.022$& $0.143\scriptstyle\pm0.112$ & $0.163\scriptstyle\pm0.032$&$\textcolor{blue}{\underline{0.120}\scriptstyle\pm0.033}$ & $0.131\scriptstyle\pm0.008$ & $0.441\scriptstyle\pm0.005$ & $0.264\scriptstyle\pm0.129$ & $0.124\scriptstyle\pm0.005$ & $0.121\scriptstyle\pm0.003$ & $\textcolor{red}{\textbf{0.102}\scriptstyle\pm0.010}$  \\ 
        ~ & 336 & $0.151\scriptstyle\pm0.017$ & $0.136\scriptstyle\pm0.002$ & $0.139\scriptstyle\pm0.002$ & $0.168\scriptstyle\pm0.010$ & $\textcolor{blue}{\underline{0.122}\scriptstyle\pm0.018}$ & - & $0.571\scriptstyle\pm0.022$ & - & $0.126\scriptstyle\pm0.008$ & $0.123\scriptstyle\pm0.001$ & $\textcolor{red}{\textbf{0.107}\scriptstyle\pm0.002}$  \\ 
        ~ & 720 & $0.149\scriptstyle\pm0.025$ & $0.140\scriptstyle\pm0.009$ & $0.149\scriptstyle\pm0.012$ & $0.164\scriptstyle\pm0.024$ & $0.136\scriptstyle\pm0.098$ & - & $0.622\scriptstyle\pm0.142$ & - & $\textcolor{blue}{\underline{0.134}\scriptstyle\pm0.004}$ & $0.144\scriptstyle\pm0.017$ & $\textcolor{red}{\textbf{0.117}\scriptstyle\pm0.019}$  \\ \midrule
        \multirow{4}{*}{Traffic-L} & 96 & $0.377\scriptstyle\pm0.024$& $0.332\scriptstyle\pm0.008$ & $0.332\scriptstyle\pm0.007$ & $\textcolor{red}{\textbf{0.228}\scriptstyle\pm0.010}$ & $0.274\scriptstyle\pm0.012$ & $0.265\scriptstyle\pm0.007$ & $0.342\scriptstyle\pm0.042$ & - & $0.234\scriptstyle\pm0.006$ & $0.231\scriptstyle\pm0.003$ & $\textcolor{blue}{\underline{0.230}\scriptstyle\pm0.010}$ \\ 
        ~ & 192 &$0.388\scriptstyle\pm0.011$& $0.326\scriptstyle\pm0.009$ & $0.350\scriptstyle\pm0.010$ & $\textcolor{red}{\textbf{0.225}\scriptstyle\pm0.012}$ & - & - & $0.354\scriptstyle\pm0.012$ & - & $0.239\scriptstyle\pm0.004$ & $0.236\scriptstyle\pm0.002$ & $\textcolor{blue}{\underline{0.234}\scriptstyle\pm0.003}$ \\ 
        ~ & 336 &$0.416\scriptstyle\pm0.028$ & $0.335\scriptstyle\pm0.010$ & $0.405\scriptstyle\pm0.011$ & $\textcolor{blue}{\underline{0.242}\scriptstyle\pm0.022}$ & - & - & $0.392\scriptstyle\pm0.006$ & - & $0.246\scriptstyle\pm0.003$ & $0.248\scriptstyle\pm0.006$ & $\textcolor{red}{\textbf{0.242}\scriptstyle\pm0.007}$  \\ 
        ~ & 720 &$0.432\scriptstyle\pm0.032$& $0.361\scriptstyle\pm0.030$ & $0.453\scriptstyle\pm0.022$ & $\textcolor{blue}{\underline{0.253}\scriptstyle\pm0.012}$ & - & - & $0.478\scriptstyle\pm0.006$ & - & $0.263\scriptstyle\pm0.001$ & $0.264\scriptstyle\pm0.006$ & $\textcolor{red}{\textbf{0.248}\scriptstyle\pm0.010}$  \\ \midrule
        \multirow{4}{*}{Weather-L} & 96 &$0.146\scriptstyle\pm0.019$& $0.144\scriptstyle\pm0.017$ & $0.279\scriptstyle\pm0.027$ & $0.145\scriptstyle\pm0.016$ & $0.176\scriptstyle\pm0.011$ & $0.139\scriptstyle\pm0.010$ & $0.113\scriptstyle\pm0.022$ & $\textcolor{blue}{\underline{0.087}\scriptstyle\pm0.012}$ & $0.164\scriptstyle\pm0.023$ & $0.145\scriptstyle\pm0.017$ & $\textcolor{red}{\textbf{0.086}\scriptstyle\pm0.011}$  \\ 
        ~ & 192 &$0.148\scriptstyle\pm0.022$& $0.145\scriptstyle\pm0.015$ & $0.264\scriptstyle\pm0.013$ & $0.144\scriptstyle\pm0.012$ & $0.166\scriptstyle\pm0.022$ & $0.160\scriptstyle\pm0.037$ & $0.144\scriptstyle\pm0.020$ & $\textcolor{blue}{\underline{0.086}\scriptstyle\pm0.007}$ & $0.158\scriptstyle\pm0.024$ & $0.147\scriptstyle\pm0.025$ & $\textcolor{red}{\textbf{0.083}\scriptstyle\pm0.011}$  \\ 
        ~ & 336 &$0.152\scriptstyle\pm0.032$& $0.146\scriptstyle\pm0.011$ & $0.283\scriptstyle\pm0.021$ & $0.149\scriptstyle\pm0.023$ & $0.168\scriptstyle\pm0.014$ & $0.170\scriptstyle\pm0.027$ & $0.138\scriptstyle\pm0.033$ & $\textcolor{blue}{\underline{0.098}\scriptstyle\pm0.002}$ & $0.162\scriptstyle\pm0.006$ & $0.160\scriptstyle\pm0.012$ & $\textcolor{red}{\textbf{0.093}\scriptstyle\pm0.010}$  \\ 
        ~ & 720 &$0.162\scriptstyle\pm0.009$& $0.147\scriptstyle\pm0.019$ & $0.317\scriptstyle\pm0.021$ & $0.152\scriptstyle\pm0.029$ & $0.149\scriptstyle\pm0.034$ & $0.148\scriptstyle\pm0.040$ & $0.141\scriptstyle\pm0.026$ &$ \textcolor{blue}{\underline{0.102}\scriptstyle\pm0.005}$ & $0.136\scriptstyle\pm0.020$ & $0.135\scriptstyle\pm0.008$ & $\textcolor{red}{\textbf{0.099}\scriptstyle\pm0.009}$  \\ \midrule
        \multirow{4}{*}{Exchange-L} & 96 &$0.079\scriptstyle\pm0.005$& $0.077\scriptstyle\pm0.001$ & $0.069\scriptstyle\pm0.007$ & $0.079\scriptstyle\pm0.002$ & $\textcolor{blue}{\underline{0.033}\scriptstyle\pm0.003}$ & $0.036\scriptstyle\pm0.009$ & $0.090\scriptstyle\pm0.010$ & $0.036\scriptstyle\pm0.005$ & $0.079\scriptstyle\pm0.002$ & $0.091\scriptstyle\pm0.009$ & $\textcolor{red}{\textbf{0.032}\scriptstyle\pm0.002}$  \\ 
        ~ & 192 &$0.081\scriptstyle\pm0.015$& $0.078\scriptstyle\pm0.008$ & $0.069\scriptstyle\pm0.007$ & $0.081\scriptstyle\pm0.002$ & $\textcolor{blue}{\underline{0.044}\scriptstyle\pm0.004}$ & $0.058\scriptstyle\pm0.007$ & $0.106\scriptstyle\pm0.010$ & $0.058\scriptstyle\pm0.005$ & $0.100\scriptstyle\pm0.019$ & $0.087\scriptstyle\pm0.005$ & $\textcolor{red}{\textbf{0.040}\scriptstyle\pm0.005}$ \\ 
        ~ & 336 &$0.086\scriptstyle\pm0.003$& $0.083\scriptstyle\pm0.005$ & $0.071\scriptstyle\pm0.005$ & $0.085\scriptstyle\pm0.010$ & $0.074\scriptstyle\pm0.017$ & $\textcolor{blue}{\underline{0.058}\scriptstyle\pm0.009}$ & $0.106\scriptstyle\pm0.010$ & $0.076\scriptstyle\pm0.006$ & $0.086\scriptstyle\pm0.008$& $0.091\scriptstyle\pm0.002$ & $\textcolor{red}{\textbf{0.054}\scriptstyle\pm0.001}$ \\ 
        ~ & 720 &$0.116\scriptstyle\pm0.022$& $0.113\scriptstyle\pm0.015$ & $\textcolor{blue}{\underline{0.097}\scriptstyle\pm0.011}$ & $0.126\scriptstyle\pm0.001$ & $0.182\scriptstyle\pm0.010$ & $0.191\scriptstyle\pm0.006$ & $0.142\scriptstyle\pm0.009$ & $0.173\scriptstyle\pm0.020$ & $0.113\scriptstyle\pm0.016$ & $0.103\scriptstyle\pm0.009$ & $\textcolor{red}{\textbf{0.084}\scriptstyle\pm0.017}$  \\ \midrule
        \multirow{4}{*}{ILI-L} & 24 &$0.303\scriptstyle\pm0.021$& $0.265\scriptstyle\pm0.027$ & $0.271\scriptstyle\pm0.032$ & $0.382\scriptstyle\pm0.018$& $0.124\scriptstyle\pm0.019$ & $\textcolor{blue}{\underline{0.118}\scriptstyle\pm0.033}$ & $0.242\scriptstyle\pm0.086$ & $0.263 \scriptstyle{\pm 0.012}$ & $0.296\scriptstyle\pm0.044$ & $0.283\scriptstyle\pm0.001$ & $\textcolor{red}{\textbf{0.116}\scriptstyle\pm0.011}$  \\
        ~ & 36 &$0.262\scriptstyle\pm0.013$& $0.222\scriptstyle\pm0.047$ & $0.258\scriptstyle\pm0.058$ & $0.286\scriptstyle\pm0.037$ & $0.144\scriptstyle\pm0.011$ & $\textcolor{blue}{\underline{0.143}\scriptstyle\pm0.089}$ & $0.246\scriptstyle\pm0.117$ & $0.298\scriptstyle\pm0.011$ & $0.298\scriptstyle\pm0.048$ & $0.307\scriptstyle\pm0.007$ & $\textcolor{red}{\textbf{0.142}\scriptstyle\pm0.008}$  \\ 
        ~ & 48 &$0.334\scriptstyle\pm0.028$& $0.262\scriptstyle\pm0.023$ & $0.225\scriptstyle\pm0.043$ & $0.291\scriptstyle\pm0.032$ & $\textcolor{blue}{\underline{0.159}\scriptstyle\pm0.020}$ & $0.160\scriptstyle\pm0.039$ & $0.275\scriptstyle\pm0.044$ & $0.301\scriptstyle\pm0.034$ & $0.320\scriptstyle\pm0.025$ & $0.314\scriptstyle\pm0.009$ & $\textcolor{red}{\textbf{0.152}\scriptstyle\pm0.017}$  \\ 
        ~ & 60 &$0.288\scriptstyle\pm0.031$& $0.278\scriptstyle\pm0.017$ & $0.245\scriptstyle\pm0.017$ & $0.287\scriptstyle\pm0.023$ & $0.216\scriptstyle\pm0.014$ & $\textcolor{blue}{\underline{0.183}\scriptstyle\pm0.019}$ & $0.272\scriptstyle\pm0.020$ & $0.299\scriptstyle\pm0.013$ & $0.325\scriptstyle\pm0.068$ & $0.333\scriptstyle\pm0.005$ & $\textcolor{red}{\textbf{0.167}\scriptstyle\pm0.007}$ \\ \bottomrule
        \multicolumn{12}{l}{Due to the excessive time and memory consumption, some results are unavailable in our implementation  and denoted as -.}
  \end{tabular}}
  \label{tab:long_term_fore_NMAE}
  \end{table}  

\clearpage
\subsection{Model Analysis}
\label{app: model analysis}
\begin{table}[!htbp]
    \centering
    \caption{Comparison on model efficiency. Lower values of Inference Speed (sec/sample) or Memory (GB) indicate higher model efficiency. L: the forecasting horizon. The results are obtained with batch size equals 1. \textcolor{red}{\textbf{Red}}: the best, \textcolor{blue}{\underline{Blue}}: the 2nd best.}
    \resizebox{0.78\columnwidth}{!}{
    \begin{tabular}{c|c|c|c|c|c|c|c|c|c|c|c}
    \toprule
        \multirow{2}{*}{Model} & \multirow{2}{*}{Metric} & Electricity-L & ETTm1-L & Exchange-S & Solar-S & Electricity-L & ETTm1-L & Electricity-L  & ETTm1-L  & Electricity-L  & ETTm1-L  \\
        ~ & ~ & (L = 720)&(L = 720)&(L = 30)&(L = 24)&(L = 96)&(L = 96)&(L = 192)&(L = 192)&(L = 336)&(L = 336)\\ \hline
        \multirow{2}{*}{TSDiff} & Inference Speed & 43.068 & \textcolor{blue}{\underline{2.841}} & 1.269 & 1.858 & 4.770 & 1.200 & 12.339 & 1.314 & 20.582 & \textcolor{blue}{\underline{1.676}}  \\
        ~ & Memory & 0.896 & 0.332 & 0.033 & 0.040 & 0.255 & 0.084 & 0.330 & 0.122 & 0.479 & 0.184  \\ \midrule
        \multirow{2}{*}{GRU NVP} & Inference Speed & \textcolor{red}{\textbf{26.296}} & 3.460 & \textcolor{blue}{\underline{0.405}} & \textcolor{blue}{\underline{0.665}} & \textcolor{blue}{\underline{3.450}} & \textcolor{blue}{\underline{0.580}} & \textcolor{red}{\textbf{7.414}} & \textcolor{blue}{\underline{1.115}} & \textcolor{red}{\textbf{12.441}} & 1.856  \\ 
        ~ & Memory & 0.427 & \textcolor{red}{\textbf{0.023}} & 0.014 & 0.040 & 0.145 & \textcolor{blue}{\underline{0.014}} & 0.173 & \textcolor{blue}{\underline{0.015}} & 0.244 & \textcolor{red}{\textbf{0.018}}  \\ \midrule
        \multirow{2}{*}{GRU MAF} & Inference Speed & 435.105 & 18.635 & 0.817 & 9.120 & 49.442 & 1.631 & 177.86 & 4.853 & 290.000 & 9.088  \\
        ~ & Memory & \textcolor{blue}{\underline{0.372}} & 0.028 & 0.013 & 0.040 & 0.129 & 0.023 & 0.175 & 0.024 & \textcolor{blue}{\underline{0.246}} & 0.025  \\ \midrule
        \multirow{2}{*}{Trans MAF} & Inference Speed & 532.151 & 19.401 & 0.883 & 9.275 & 45.367 & 1.900 & 169.368 & 5.336 & 311.954 & 10.130  \\ 
        ~ & Memory & \textcolor{red}{\textbf{0.368}} & 0.081 & \textcolor{blue}{\underline{0.011}} & \textcolor{blue}{\underline{0.037}} & 0.147 & 0.073 & 0.201 & 0.076 & 0.272 & 0.075  \\ \midrule
        \multirow{2}{*}{TimeGrad} & Inference Speed & - & - & 24.896 & 19.641 & 113.103 & 94.888 & 142.104 & 155.013 & - & 284.951  \\
        ~ & Memory & - & - & 0.016 & 0.041 & \textcolor{blue}{\underline{0.128}} & 0.016 & \textcolor{red}{\textbf{0.149}} & 0.022 & - & 0.034  \\ \midrule
        \multirow{2}{*}{CSDI} & Inference Speed & - & 86.182 & 19.251 & 29.251 & 388.315 & 16.328 & 659.428 & 25.838 & - & 39.883  \\
        ~ & Memory & - & 0.133 & 0.182 & 0.723 & 1.411 & 0.027 & 3.024 & 0.033 & - & 0.051  \\ \midrule
        \multirow{2}{*}{K2VAE} & Inference Speed & \textcolor{blue}{\underline{39.834}} & \textcolor{red}{\textbf{0.998}} & \textcolor{red}{\textbf{0.309}} & \textcolor{red}{\textbf{0.483}} & \textcolor{red}{\textbf{3.310}} & \textcolor{red}{\textbf{0.257}} & \textcolor{blue}{\underline{8.836}} & \textcolor{red}{\textbf{0.374}} & \textcolor{blue}{\underline{17.961}} & \textcolor{red}{\textbf{0.475}}  \\
        ~ & Memory & 0.474 & \textcolor{blue}{\underline{0.028}} & \textcolor{red}{\textbf{0.011}} & \textcolor{red}{\textbf{0.017}} & \textcolor{red}{\textbf{0.094}} & \textcolor{red}{\textbf{0.013}} & \textcolor{blue}{\underline{0.154}} & \textcolor{red}{\textbf{0.015}} & \textcolor{red}{\textbf{0.240}} & \textcolor{blue}{\underline{0.019}} \\ \bottomrule
        \multicolumn{8}{l}{Due to the excessive time and memory consumption, some results are unavailable in our implementation  and denoted as -.}
    \end{tabular}}
    \label{tab: complete model efficiency}
\end{table}

\begin{table}[!htbp]
    \centering
    \caption{Comparison on different Koopman Operators. Lower CRPS or NMAE values indicate better performance. \textcolor{red}{\textbf{Red}}: the best. L: the forecasting horizon. }
    \resizebox{\columnwidth}{!}{
        \begin{tabular}{c|c|cccccccccc}
    \toprule
       Koopman & \multirow{2}{*}{Metrics} & Electricity-L & ETTm1-L & Exchange-S & Solar-S & Electricity-L  & ETTm1-L &Electricity-L  & ETTm1-L & Electricity-L & ETTm1-L \\ 
         Operator & &(L = 720) &(L = 720) &(L = 30) &(L = 24)&(L = 96)&(L = 96)&(L = 192)&(L = 192)&(L = 336)&(L = 336)\\\hline
        \multirow{2}{*}{$\mathcal{K}_{loc}$} & CRPS & - & - & $0.012\scriptstyle\pm0.002$ & $0.450\scriptstyle\pm0.012$ & $0.077\scriptstyle\pm0.003$ & $0.298\scriptstyle\pm0.022$ & $0.114\scriptstyle\pm0.008$ & - & - & -  \\ 
        ~ & NMAE & - & - & $0.014\scriptstyle\pm0.002$ & $0.566\scriptstyle\pm0.015$ & $0.101\scriptstyle\pm0.004$ & $0.387\scriptstyle\pm0.014$ & $0.134\scriptstyle\pm0.011$ & - & - & -  \\ \midrule
        \multirow{2}{*}{$\mathcal{K}_{glo}$} & CRPS & $0.065\scriptstyle\pm0.007$ & $0.311\scriptstyle\pm0.024$ & $0.011\scriptstyle\pm0.001$ & $0.374\scriptstyle\pm0.004$ & $0.079\scriptstyle\pm0.004$ & $0.248\scriptstyle\pm0.016$ & $0.082\scriptstyle\pm0.004$ & $0.263\scriptstyle\pm0.012$ & $0.057\scriptstyle\pm0.002$ & $0.268\scriptstyle\pm0.026$  \\ 
        ~ & NMAE & $0.130\scriptstyle\pm0.024$ & $0.395\scriptstyle\pm0.027$ & $0.013\scriptstyle\pm0.001$ & $0.488\scriptstyle\pm0.008$ & $0.109\scriptstyle\pm0.005$ & $0.304\scriptstyle\pm0.021$ & $0.106\scriptstyle\pm0.012$ & $0.329\scriptstyle\pm0.017$ & $0.114\scriptstyle\pm0.003$ & $0.341\scriptstyle\pm0.020$  \\ \midrule
        \multirow{2}{*}{$\mathcal{K}_{loc}+\mathcal{K}_{glo}$} & CRPS & $\textcolor{red}{\textbf{0.057}\scriptstyle\pm 0.005}$ & $\textcolor{red}{\textbf{0.294}\scriptstyle\pm0.026}$ & $\textcolor{red}{\textbf{0.009}\scriptstyle\pm0.001}$ & $\textcolor{red}{\textbf{0.367}\scriptstyle\pm0.005}$ &$\textcolor{red}{\textbf{0.073}\scriptstyle\pm0.002}$ & $\textcolor{red}{\textbf{0.232}\scriptstyle\pm0.010}$ & $\textcolor{red}{\textbf{0.080}\scriptstyle\pm0.004}$ & $\textcolor{red}{\textbf{0.259}\scriptstyle\pm0.013}$ & $\textcolor{red}{\textbf{0.054}\scriptstyle\pm0.001}$ & $\textcolor{red}{\textbf{0.262}\scriptstyle\pm0.030}$  \\ 
        ~ & NMAE & $\textcolor{red}{\textbf{0.117}\scriptstyle\pm 0.019}$ & $\textcolor{red}{\textbf{0.373}\scriptstyle\pm0.032}$ & $\textcolor{red}{\textbf{0.009}\scriptstyle\pm0.001}$ & $\textcolor{red}{\textbf{0.480}\scriptstyle\pm0.008}$ & $\textcolor{red}{\textbf{0.093}\scriptstyle\pm0.002}$ & $\textcolor{red}{\textbf{0.284}\scriptstyle\pm0.011}$ & $\textcolor{red}{\textbf{0.102}\scriptstyle\pm0.010}$ & $\textcolor{red}{\textbf{0.323}\scriptstyle\pm0.020}$ & $\textcolor{red}{\textbf{0.107}\scriptstyle\pm0.002}$ & $\textcolor{red}{\textbf{0.330} \scriptstyle\pm0.014}$ \\ \bottomrule
        \multicolumn{12}{l}{Due to the numerical instability, some results are unavailable in our implementation  and denoted as -.}
    \end{tabular}
    }
    \label{tab: complete koopman operator}
\end{table}

\begin{table}[!htbp]
    \centering
    \caption{Comparison on different connections of KalmanNet. Lower CRPS or NMAE values indicate better performance. \textcolor{red}{\textbf{Red}}: the best. L: the forecasting horizon.}
    \resizebox{\columnwidth}{!}{
    \begin{tabular}{c|c|cccccccccc}
    \toprule
        Connections  & \multirow{2}{*}{Metrics} & Electricity-L & ETTm1-L & Exchange-S & Solar-S & Electricity-L & ETTm1-L & Electricity-L  & ETTm1-L & Electricity-L  & ETTm1-L  \\ 
        in KalmanNet & 
        & (L = 720) & (L = 720) &  (L = 30) & (L = 24) & (L = 96) & (L = 96) &(L = 192) &(L = 192) & (L = 336) & (L = 336) \\ \hline
        \multirow{2}{*}{w/o AuxiliaryNet} & CRPS & $0.082\scriptstyle\pm0.011$& $0.359\scriptstyle\pm0.024$ & $0.015\scriptstyle\pm0.002$ & $0.398\scriptstyle\pm0.005$ & $0.083\scriptstyle\pm0.001$ & $0.268\scriptstyle\pm0.014$ &$ 0.107\scriptstyle\pm0.008$ & $0.263\scriptstyle\pm0.014$ & $0.074\scriptstyle\pm0.004$ & $0.267\scriptstyle\pm0.025$  \\ 
        ~ & NMAE & $0.188\scriptstyle\pm0.028$ & $0.442\scriptstyle\pm0.029$ & $0.022\scriptstyle\pm0.001$ & $0.531\scriptstyle\pm0.010$ & $0.099\scriptstyle\pm0.003$ & $0.348\scriptstyle\pm0.011$ & $0.142\scriptstyle\pm0.012$ & $0.336\scriptstyle\pm0.016$ & $0.147\scriptstyle\pm0.003$ &$ 0.346\scriptstyle\pm0.018$  \\ \midrule
        \multirow{2}{*}{w/o skip connection} & CRPS & $0.063\scriptstyle\pm0.007$ &$ 0.315\scriptstyle\pm0.016$ & $0.011\scriptstyle\pm0.001$ & $0.388\scriptstyle\pm0.006$ & $0.092\scriptstyle\pm0.004$ & $0.243\scriptstyle\pm0.012$ & $0.087\scriptstyle\pm0.002$ & $0.277\scriptstyle\pm0.013$ & $0.058\scriptstyle\pm0.001$ &$0.266\scriptstyle\pm0.019$  \\ 
        ~ & NMAE & $0.131\scriptstyle\pm0.015$ & $0.402\scriptstyle\pm0.030$ & $0.013\scriptstyle\pm0.004$ & $0.511\scriptstyle\pm0.008$ & $0.116\scriptstyle\pm0.003$ & $0.292\scriptstyle\pm0.010$ & $0.114\scriptstyle\pm0.007$ & $0.376\scriptstyle\pm0.022$ & $0.119\scriptstyle\pm0.001$ & $0.349\scriptstyle\pm0.022$  \\ \midrule
        \multirow{2}{*}{w/o control input} & CRPS & $0.069\scriptstyle\pm0.005$ & $0.322\scriptstyle\pm0.017$ & $0.013\scriptstyle\pm0.006$ & $0.423\scriptstyle\pm0.005$ & $0.079\scriptstyle\pm0.002$ & $0.242\scriptstyle\pm0.011$ &$ 0.084\scriptstyle\pm0.003$ & $0.269\scriptstyle\pm0.017$ & $0.064\scriptstyle\pm0.001$ & $0.271\scriptstyle\pm0.026$  \\ 
        ~ & NMAE & $0.142\scriptstyle\pm0.018$ & $0.418\scriptstyle\pm0.022$ & $0.017\scriptstyle\pm0.003$ & $0.560\scriptstyle\pm0.009$ & $0.104\scriptstyle\pm0.001$ & $0.295\scriptstyle\pm0.012$ & $0.108\scriptstyle\pm0.007$ & $0.341\scriptstyle\pm0.019$ & $0.128\scriptstyle\pm0.002$ & $0.356\scriptstyle\pm0.015$  \\ \midrule
        \multirow{2}{*}{Mixed}& CRPS & $\textcolor{red}{\textbf{0.057}\scriptstyle\pm 0.005}$ & $\textcolor{red}{\textbf{0.294}\scriptstyle\pm0.026}$ & $\textcolor{red}{\textbf{0.009}\scriptstyle\pm0.001}$ & $\textcolor{red}{\textbf{0.367}\scriptstyle\pm0.005}$ &$\textcolor{red}{\textbf{0.073}\scriptstyle\pm0.002}$ & $\textcolor{red}{\textbf{0.232}\scriptstyle\pm0.010}$ & $\textcolor{red}{\textbf{0.080}\scriptstyle\pm0.004}$ & $\textcolor{red}{\textbf{0.259}\scriptstyle\pm0.013}$ & $\textcolor{red}{\textbf{0.054}\scriptstyle\pm0.001}$ & $\textcolor{red}{\textbf{0.262}\scriptstyle\pm0.030}$  \\ 
        ~ & NMAE & $\textcolor{red}{\textbf{0.117}\scriptstyle\pm 0.019}$ & $\textcolor{red}{\textbf{0.373}\scriptstyle\pm0.032}$ & $\textcolor{red}{\textbf{0.009}\scriptstyle\pm0.001}$ & $\textcolor{red}{\textbf{0.480}\scriptstyle\pm0.008}$ & $\textcolor{red}{\textbf{0.093}\scriptstyle\pm0.002}$ & $\textcolor{red}{\textbf{0.284}\scriptstyle\pm0.011}$ & $\textcolor{red}{\textbf{0.102}\scriptstyle\pm0.010}$ & $\textcolor{red}{\textbf{0.323}\scriptstyle\pm0.020}$ & $\textcolor{red}{\textbf{0.107}\scriptstyle\pm0.002}$ & $\textcolor{red}{\textbf{0.330} \scriptstyle\pm0.014}$ \\ \bottomrule
    \end{tabular}
    }
    \label{tab: complete kalman filter}
\end{table}

\begin{table}[!htbp]
    \centering
    \caption{Ablations on KoopmanNet and KalmanNet. Lower CRPS or NMAE values indicate better performance. \textcolor{red}{\textbf{Red}}: the best. L: the forecasting horizon.}
    \resizebox{\columnwidth}{!}{
    \begin{tabular}{c|c|cccccccccc}
    \toprule
   \multirow{2}{*}{Variants}  & \multirow{2}{*}{Metrics} & Electricity-L & ETTm1-L & Exchange-S & Solar-S & Electricity-L & ETTm1-L & Electricity-L  & ETTm1-L & Electricity-L  & ETTm1-L  \\ 
         & 
        & (L = 720) & (L = 720) &  (L = 30) & (L = 24) & (L = 96) & (L = 96) &(L = 192) &(L = 192) & (L = 336) & (L = 336) \\ \hline
        \multirow{2}{*}{w/o KoopmanNet} & CRPS & $0.074\scriptstyle\pm0.009$ & $0.443\scriptstyle\pm0.034$ & $0.014\scriptstyle\pm0.002$ & $0.385\scriptstyle\pm0.008$ & $0.079\scriptstyle\pm0.003$ & $0.265\scriptstyle\pm0.018$ & $0.087\scriptstyle\pm0.006$ & $0.288\scriptstyle\pm0.019$ & $0.114\scriptstyle\pm0.006$ & $0.291\scriptstyle\pm0.037$  \\
        ~ & NMAE & $0.162\scriptstyle\pm0.015$ & $0.601\scriptstyle\pm0.058$ & $0.016\scriptstyle\pm0.001$ & $0.528\scriptstyle\pm0.014$ & $0.112\scriptstyle\pm0.006$ & $0.328\scriptstyle\pm0.022$ & $0.112\scriptstyle\pm0.012$ & $0.382\scriptstyle\pm0.027$ & $0.175\scriptstyle\pm0.009$ & $0.372\scriptstyle\pm0.027$  \\ \midrule
        \multirow{2}{*}{w/o KalmanNet} & CRPS & $0.089\scriptstyle\pm0.011$ & $0.398\scriptstyle\pm0.038$ & $0.011\scriptstyle\pm0.001$ & $0.375\scriptstyle\pm0.005$ & $0.098\scriptstyle\pm0.004$ & $0.278\scriptstyle\pm0.023$ & $0.091\scriptstyle\pm0.003$ & $0.375\scriptstyle\pm0.025$ & $0.266\scriptstyle\pm0.012$ & $0.301\scriptstyle\pm0.035$  \\
        ~ & NMAE & $0.192\scriptstyle\pm0.023$ & $0.539\scriptstyle\pm0.044 $& $0.012\scriptstyle\pm0.001$ & $0.499\scriptstyle\pm0.009$ &$ 0.133\scriptstyle\pm0.008$ & $0.338\scriptstyle\pm0.012$ & $0.122\scriptstyle\pm0.011$ & $0.443\scriptstyle\pm0.033$ & $0.359\scriptstyle\pm0.017$ & $0.394\scriptstyle\pm0.018$  \\ \midrule
        \multirow{2}{*}{$K^2$VAE} & CRPS & $\textcolor{red}{\textbf{0.057}\scriptstyle\pm 0.005}$ & $\textcolor{red}{\textbf{0.294}\scriptstyle\pm0.026}$ & $\textcolor{red}{\textbf{0.009}\scriptstyle\pm0.001}$ & $\textcolor{red}{\textbf{0.367}\scriptstyle\pm0.005}$ &$\textcolor{red}{\textbf{0.073}\scriptstyle\pm0.002}$ & $\textcolor{red}{\textbf{0.232}\scriptstyle\pm0.010}$ & $\textcolor{red}{\textbf{0.080}\scriptstyle\pm0.004}$ & $\textcolor{red}{\textbf{0.259}\scriptstyle\pm0.013}$ & $\textcolor{red}{\textbf{0.054}\scriptstyle\pm0.001}$ & $\textcolor{red}{\textbf{0.262}\scriptstyle\pm0.030}$  \\ 
        ~ & NMAE & $\textcolor{red}{\textbf{0.117}\scriptstyle\pm 0.019}$ & $\textcolor{red}{\textbf{0.373}\scriptstyle\pm0.032}$ & $\textcolor{red}{\textbf{0.009}\scriptstyle\pm0.001}$ & $\textcolor{red}{\textbf{0.480}\scriptstyle\pm0.008}$ & $\textcolor{red}{\textbf{0.093}\scriptstyle\pm0.002}$ & $\textcolor{red}{\textbf{0.284}\scriptstyle\pm0.011}$ & $\textcolor{red}{\textbf{0.102}\scriptstyle\pm0.010}$ & $\textcolor{red}{\textbf{0.323}\scriptstyle\pm0.020}$ & $\textcolor{red}{\textbf{0.107}\scriptstyle\pm0.002}$ & $\textcolor{red}{\textbf{0.330} \scriptstyle\pm0.014}$ \\ \bottomrule
    \end{tabular}}
        \label{tab: complete koopman kalman}
\end{table}

\clearpage
\subsection{Showcases}
We provide some showcases of $K^2$VAE in Figure~\ref{fig: Solar-S cases}, \ref{fig: ETTm1-L cases}, and \ref{fig: Elc-L cases}, which demonstrates the strong interval estimation capabilities of $K^2$VAE. We observe that $K^2$VAE achieves good performance in 95\% confidence interval, which means the forecasting horizon of the time series is well modeled and estimated.
\begin{figure}[!htbp]
  \centering
  \includegraphics[width=0.63\linewidth]{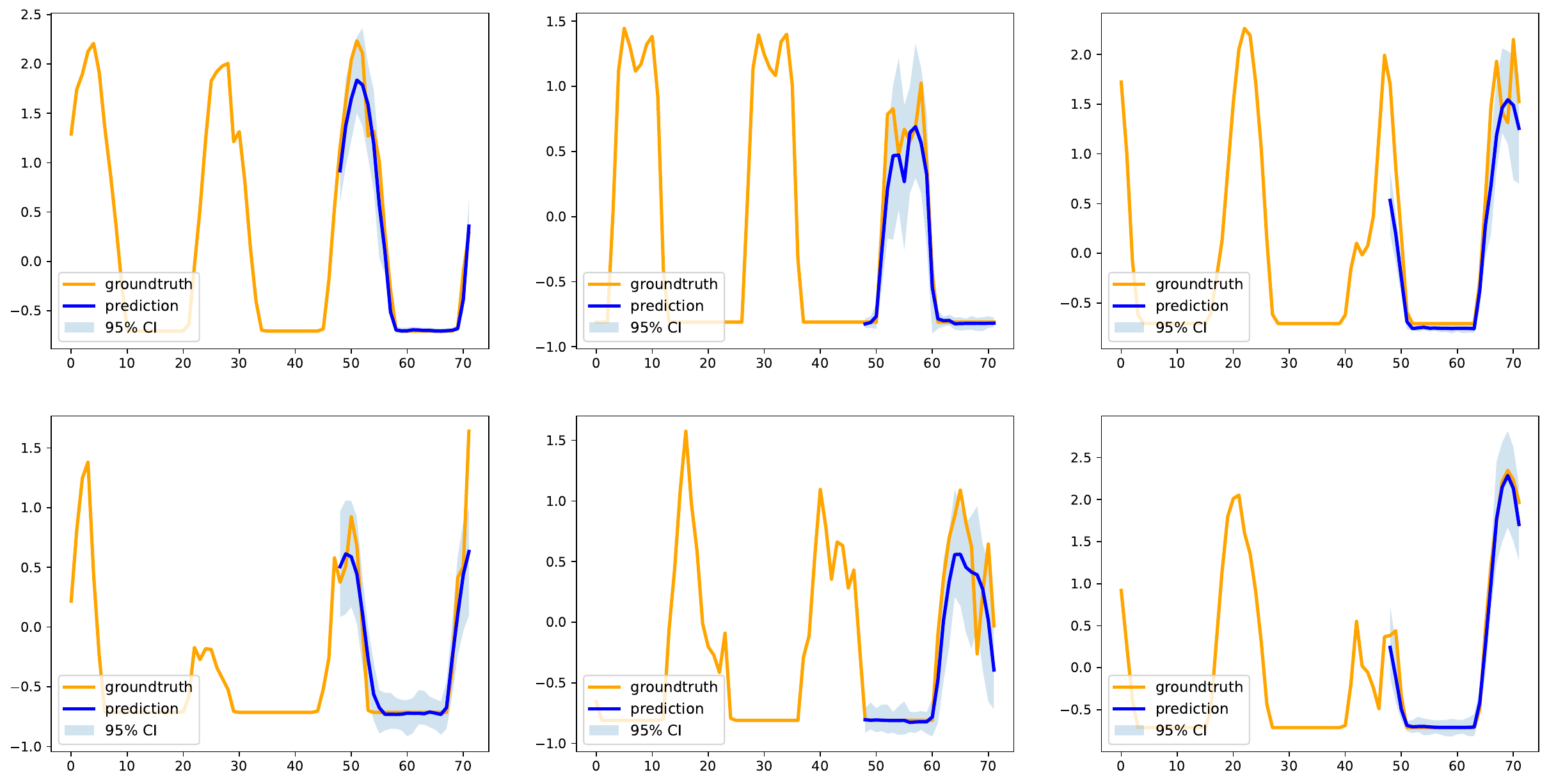}
  \caption{Visualization of input-24-predict-24 results on the Solar-S dataset.}
  \label{fig: Solar-S cases}
\end{figure}
\begin{figure}[!htbp]
  \centering
  \includegraphics[width=0.63\linewidth]{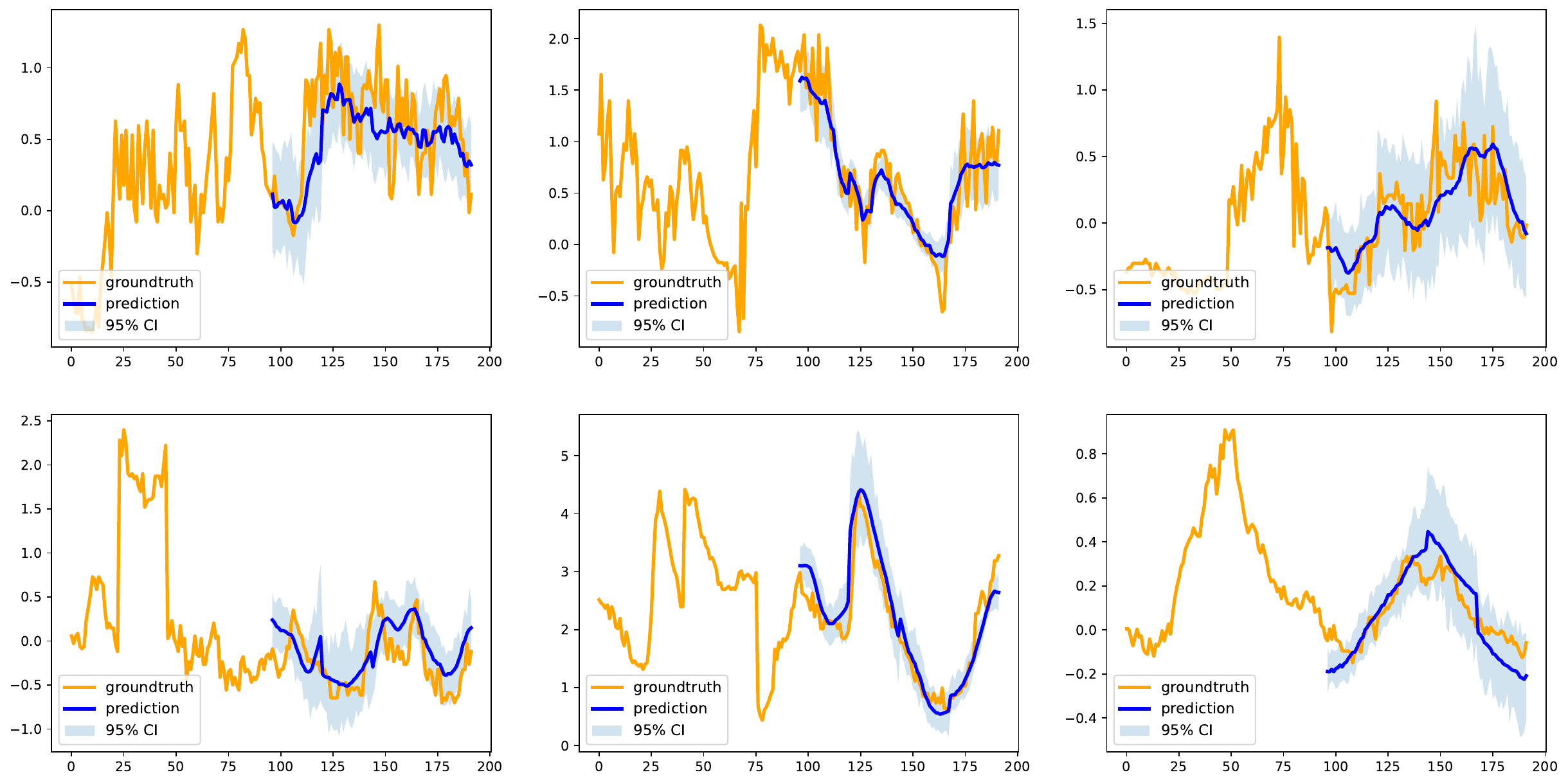}
  \caption{Visualization of input-96-predict-96 results on the ETTm1-L dataset.}
  \label{fig: ETTm1-L cases}
\end{figure}
\begin{figure}[!htbp]
  \centering
  \includegraphics[width=0.63\linewidth]{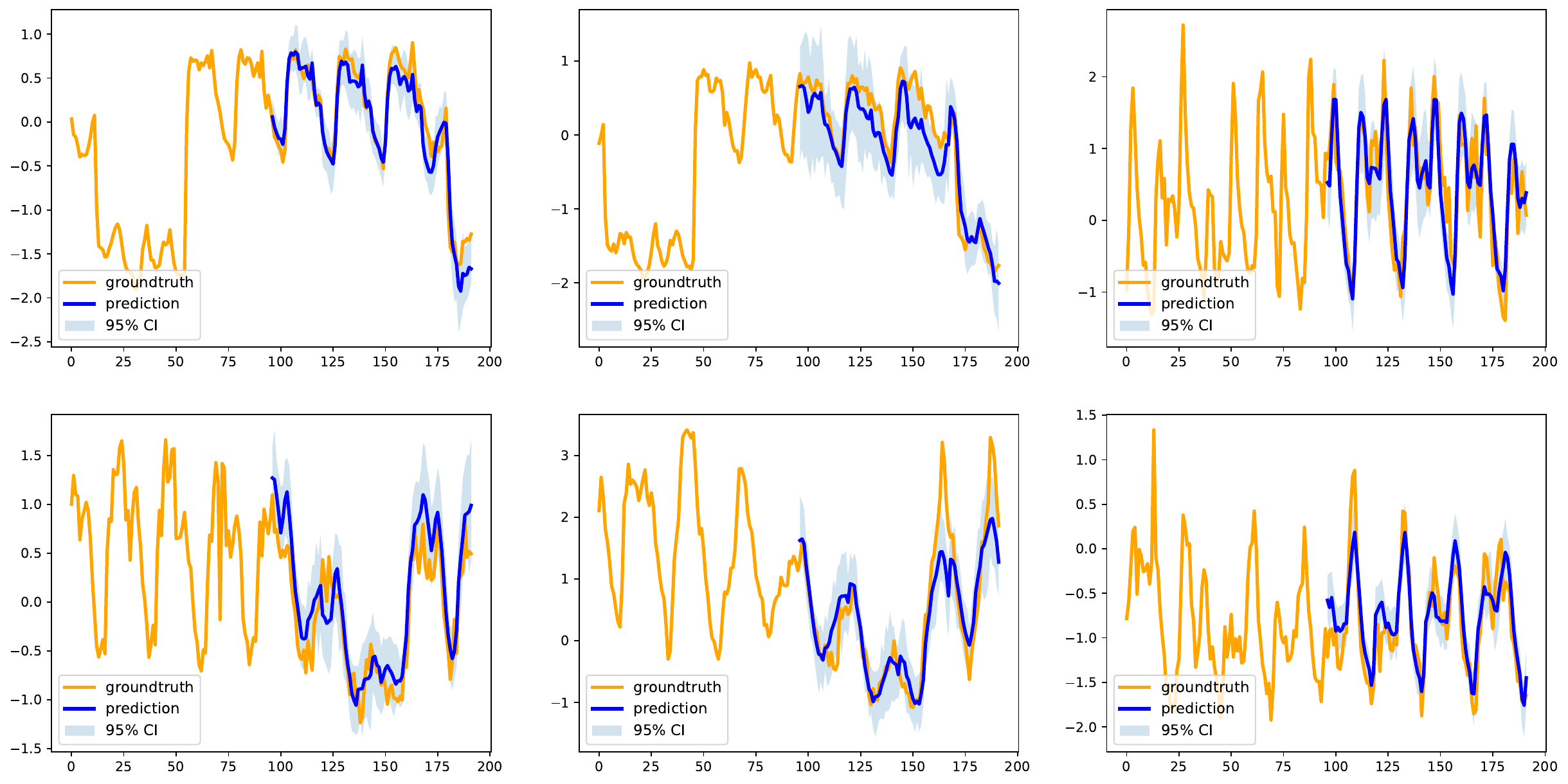}
  \caption{Visualization of input-96-predict-96 results on the Electricity-L dataset.}
  \label{fig: Elc-L cases}
\end{figure}
%%%%%%%%%%%%%%%%%%%%%%%%%%%%%%%%%%%%%%%%%%%%%%%%%%%%%%%%%%%%%%%%%%%%%%%%%%%%%%%
%%%%%%%%%%%%%%%%%%%%%%%%%%%%%%%%%%%%%%%%%%%%%%%%%%%%%%%%%%%%%%%%%%%%%%%%%%%%%%%

\end{document}